\documentclass[10pt]{article}
\usepackage[normalem]{ulem}
\usepackage{fullpage} 
\usepackage{amsmath,amsfonts,amscd,amssymb,bm}
\usepackage{cite,algorithm,algorithmic,enumerate,multirow}
\usepackage{psfrag}
\usepackage[dvips]{graphicx}
\usepackage[small,bf]{caption}
\usepackage[usenames,dvipsnames,dvips]{color}
\definecolor{darkred}{RGB}{150,0,0}
\definecolor{darkgreen}{RGB}{0,150,0}
\definecolor{darkblue}{RGB}{0,0,200}
\definecolor{orange}{RGB}{205, 140,0}

\newcommand{\dgreen}[1]{\color{darkgreen}{ #1}\color{black}}
\usepackage[pagebackref]{hyperref}
\usepackage{breakurl}
\hypersetup{colorlinks=true, linkcolor=darkred, citecolor=darkgreen, urlcolor=darkblue}
\numberwithin{equation}{section}

\newtheorem{theorem}{Theorem}%[section]
\newtheorem{lemma}[theorem]{Lemma}
\newtheorem{corollary}[theorem]{Corollary}
\newtheorem{proposition}[theorem]{Proposition}

\newtheorem{remark}{Remark}
\def \endprf{\hfill {\vrule height6pt width6pt depth0pt}\medskip}
\newenvironment{proof}{\noindent {\it Proof.} }{\endprf\par}
\newcounter{example}
\newenvironment{example}[1][]{\refstepcounter{example}\par\medskip
   \noindent\textbf{Example~\theexample #1 } \rmfamily}{\medskip} 
\newcommand{\beq}{\begin{equation}}
\newcommand{\eeq}{\end{equation}}
\newcommand{\beqa}{\begin{equation} \begin{aligned}}
\newcommand{\eeqa}{\end{aligned} \end{equation}}
\newcommand{\beqas}{\begin{equation*} \begin{aligned}}
\newcommand{\eeqas}{\end{aligned} \end{equation*}}

\newcommand{\bit}{\begin{itemize}}
\newcommand{\eit}{\end{itemize}}

\newcommand{\st}{\star}
\newcommand{\eps}{\epsilon}%{\varepsilon}
\newcommand{\E}{\operatorname{\mathbb{E}}}  % expected value
\newcommand{\prob}{\operatorname{\mathbb{P}}}

\newcommand{\norm}[1]{\|#1\|}
\newcommand{\abs}[1]{|#1|}

\newcommand{\Cc}{\mathcal{C}}
\newcommand{\Yc}{\mathcal{Y}}

\newcommand{\iprod}[2]{\langle #1 , #2 \rangle}
\usepackage{array}      % for advanced column specification: use >{\command} for
\newcommand{\ER}{{Erd\H{o}s-R\'enyi~}}%\newcommand{\ER}{{Erdos-R\'enyi~}}
\newcommand{\ber}{\operatorname{Ber}}
\newcommand{\bin}{\operatorname{Bin}}
\newcommand{\fdiv}[3]{D_{#1}(#2,#3)}
\newcommand{\tri}[3]{(#1,\,#2,\,#3)}
\newcommand{\chidiv}{\widetilde D}
\newcommand{\kldiv}{D_{\mathrm{KL}}}
\title{Relative Density and Exact Recovery in Heterogeneous Stochastic Block Models }
\author{Amin Jalali, Qiyang Han, Ioana Dumitriu, Maryam Fazel}
\begin{document}
\maketitle

\begin{abstract}
The Stochastic Block Model (SBM) is a widely used random graph model for networks with communities. 
Despite the recent burst of interest in recovering communities in the SBM %(given an observation of the random graph) 
from statistical and computational points of view, there are still gaps in understanding the fundamental information theoretic and computational limits of recovery. 
%information theoretic limits of recoverability, as well as identifying algorithms that work down to these limits. 
In this paper, we consider the SBM in its full generality, where there is no restriction on the number and sizes of communities or how they grow with the number of nodes, as well as on the connection probabilities 
%the probabilities the nodes use to connect to each other
inside or across communities. 
%communities of arbitrary sizes and number which can be growing at any rate with the number of nodes, which can connect among themselves or across communities with different probabilities. 
%We study the problem of exact model recovery from a single observation of the random graph. 
This generality allows us to move past the artifacts of homogenous SBM, and understand the right parameters (such as the relative densities of communities) that define the various recovery thresholds. We outline the implications of our generalizations via a set of illustrative examples. 
For instance, $\log n$ is considered to be the standard lower bound on the cluster size for exact recovery via convex methods, for homogenous SBM. We show that it is possible, in the right circumstances (when sizes are spread and the smaller the cluster, the denser), to recover very small clusters (up to $\sqrt{\log n}$ size), if there are just a few of them (at most polylogarithmic in $n$). 
%\maryam{anything else worth mentioning here?}
\end{abstract}

\section{Introduction} 

A fundamental problem in network science and machine learning is to discover structures in large, complex real-world networks (e.g., biological, social, or information networks).  Communities are one of the most basic structures to look for, and are useful in many ways including simplifying network analysis.  
Community or cluster detection also arises in machine learning and underlies many decision tasks, as a basic step that uses pairwise relations between data points in order to understand more global structures in the data.  Applications of community detection are numerous, and include recommendation systems \cite{xu2014jointly}, image segmentation \cite{shi2000normalized, meila2001random}, learning gene network structures in bioinformatics, e.g., in protein detection \cite{CY:06} and population genetics \cite{JTZ:04}.
 
In spite of a long history of heuristic algorithms (see, e.g., \cite{leskovec2010empirical} for an empirical overview), as well as strong research interest in recent years on the theoretical side as reviewed in the next section, there are still gaps in understanding the fundamental information theoretic limits of recoverability (i.e., if there is enough information 
to reveal the communities)  and computational tractability (if there are efficient algorithms to recover them). 
This is particularly true in the case of sparse graphs (that test the limits of recoverability), %\amin{maybe ok, because we are not technical yet: because relatively dense graphs, $\Omega(\log n)$, are at the limits of exact recoverability, and sparse graphs are at the limits of weak recoverability}  
%\maryam{(maybe it's better to say: in the case of graphs with densities that test the limits of recoverability (so that we refer to both exact and weak recovery), although I'm OK with the sentence as is too)}\amin{no mention of "density" yet..}
graphs with heterogeneous communities (communities varying greatly in size and connectivity), graphs with a number of communities that grows with the number of nodes, and partially observed graphs (with various observation models). 

In this paper, we 
%refine the understanding of recovery regimes and analyze algorithms for community detection \maryam{actually this sounds awkward}
study recovery regimes and algorithms for community detection in sparse graphs generated under a heterogeneous stochastic block model, where there is no restriction on the number and sizes of communities or how they grow with the number of nodes, as well as the connection probabilities inside or across communities. 
%We identify key network descriptors or parameters that we call relative densities (defined in \eqref{def:relative_density})  that govern the exact recoverability of the communities, and ranges of these parameters
%descriptors (later we've used "parameters", not "descriptors")
We propose key network descriptors, called relative densities (defined in \eqref{def:relative_density}), that govern the exact recoverability of the communities, and determine ranges of these parameters
that lead to various regimes of difficulty of recovery.  
%\red{RH: How about this: In this paper, we propose the descriptors \emph{relative densities} (defined in (\ref{def:relative_density})) that describes various exactly recoverable regimes under algorithms of varying computational difficulty, along with information limit of community detection in sparse graphs generated under a heterogenous stochastic block model.}  
%Our analysis handles heterogeneous communities of any size and density, with possibly growing number of communities, as well as partial observations of the random graph. 
The implications of our generalizations are outlined in Section \ref{sec:this-paper} where illustrative examples 
provide insight into our results in Section \ref{sec:main-results}. 

\subsection{The Heterogenous Stochastic Block Model and Exact Recovery}	\label{sec:GSBM-def}

The stochastic block model (SBM), first introduced and studied in mathematical sociology by
Holland, Laskey and Leinhardt in 1983 \cite{holland1983stochastic}, can be described as follows. 
Start with $n$ vertices and partition the vertex set $\{1,2,\ldots,n\}$
into $r$ groups $V_1, V_2,\ldots, V_r\,$, of sizes $n_1, n_2,\ldots,
n_r$ respectively. 
%\old{\amin{how about replacing the following sentences with the text coming after the equation? I'm suggesting because while inside the clusters, we're \ER, outside, we need the equation to characterize the scenario; then, why not using it from the beginning which eliminates the need for defining \ER as well.} Each of the $V_i$'s is then endowed with an \ER graph structure $\mathcal{G}(n_i, p_i)$ (within each cluster $V_i\,$, the probability of an edge is given by the local probability $p_i$). In addition, inter-cluster edges (also known as ``ambient" edges) have probability $q < \min_k p_k$ (note that this is necessary in order for the idea of clusters to make sense). More specifically, the probability of an edge between vertices $i$ and $j$ (the existence of which is indicated by $i \sim j$) is given by
%\[
%\prob(i \sim j) = \left \{ \begin{array}{cl} p_k, & \mbox{if there is a $k \in \{1,2,\ldots,r\}$ such that $i,j \in V_k$}, \\
%q, & \mbox{otherwise~.} \end{array} \right .
%\] }
Then, we draw an edge between two nodes with a probability depending on which communities they belong to; i.e., the probability of an edge between vertices $i$ and $j$ (denoted by $i \sim j$) is given by
\begin{align}	\label{eq:rand-graph-dist}
\prob(i \sim j) = 
\begin{cases}
p_k 	& \text{if there is a $k \in \{1,2,\ldots,r\}$ such that $i,j \in V_k$} \\
q	& \text{otherwise}
\end{cases}
\end{align}
where we assume $q < \min_k p_k$ in order for the idea of communities to make sense. Such inter-cluster edges are also known as ``ambient" edges. Notice that each of the $V_k$'s is endowed with an \ER graph structure $\mathcal{G}(n_k, p_k)$ (within each community $V_k\,$, the probability of an edge is given by the local probability $p_k$). This defines a distribution over random graphs known as the stochastic block model. To contrast our study of this general setting with previous works where homogenous SBM is considered (where the sizes and probabilities associated to the communities are equal, e.g., in \cite{chen2014statistical}), or other special cases of SBMs are studied (e.g., when the number of communities is fixed or grows slowly with the number of nodes such as in \cite{abbe2015community}), we sometimes refer to the above model as the {\em heterogenous stochastic block model}.

The community detection problem studied in this paper is then stated simply as: given the adjacency matrix of a graph generated by the heterogenous stochastic block model, can we recover the labels of {\em all} vertices, with high probability, using an algorithm that has been proved to do so, whether in polynomial time or not. 
Note that recovery with high probability is the best one can hope for, as--with tiny probability--the model can generate graphs where the partition is unrecoverable, e.g., the complete graph. Whether this problem is solvable depends on the parameters involved, and our results characterize parts of the model space for which such recovery is possible. Moreover, based on the computational complexity of the proposed algorithm, we can be in different subregimes, 
hard (recovery is possible \emph{theoretically}, but not necessarily \emph{efficiently}), 
easy (recovery can be done efficiently; i.e., there is a polynomial-time algorithm), 
and simple (recovery can be done by simple counting and thresholding procedures), 
as explained in the next section. 
%{all regimes can be divided into those subregimes}

In the next subsection, we mention other natural questions in community detection and review existing results in the literature. We summarize our new results in section \ref{sec:this-paper}.
%An experienced reader might skip this section and go to Section \ref{sec:this-paper} for the new results. 

\subsection{Related Work} 
%The {\em quality of inference} for the underlying community structure from a single draw of the random graph can fall into %different categories depending on the regime of model parameters. 
What we can infer about the community structure from a single draw of the random graph varies based on the regime of model parameters. Often, the following scenarios are considered.
%roughly speaking, one can identify the following regimes. \amin{why ``roughly speaking"?}

%\maryam{in the list below, should say what we mean by 'symmetric'}

\begin{enumerate}[1.]
\item {\em Exact Recovery (Strong Consistency).} In this regime it is possible to recover all labels, with high probability. That is, an algorithm has been proved to do so, whether in polynomial time or not. 
%Furthermore, depending on the computational complexity of the recovery algorithm, exact recovery might be categorized into hard, easy, or simple subregimes, as explained in Section \ref{sec:thresholds_review}. %and illustrated in Figure \ref{fig:4regimes-GSBM}.
%
Notice that we need the nodes in all communities to be connected for the exact recovery to be possible.  %\amin{we need communities, but at most one community, to be connected as subgraphs. }\maryam{sorry, I don't get this sentence...}
%\maryam{actually, what does this mean? I guess isolated nodes can be ignored (treated like outliers)?}

\item {\em Almost Exact Recovery (Weak Consistency).} A total of $n - o(n)$ labels are
  recoverable, but no more. For example, consider the case where the graph has multiple components, all but one of which are tiny; the tiny components cannot be correctly classified. 

%\red{From AbbeSandon15: ``One may ask for the finer question of how much can be recovered about the communities. For a given set of parameters of the block model, finding the proportion of nodes (as a function of p and Q) that can be correctly recovered (whp) is an open problem. Obtaining a closed form formula for this question is unlikely, even in the symmetric case with two communities. Partial results were obtained in [MNSa] for two-symmetric communities, but the general problem remains open even for determining scaling laws. One may also consider the special case of partial recovery where only an o(n) fraction of nodes is allowed to be mis-classified (whp), called almost exact recovery or weak consistency, but no sharp phase transition is to be expected for this requirement."}

\item {\em Partial Recovery or Approximation Regime.} 
%Only a fraction, i.e., $(1-O(1))n$, of vertices can be guaranteed to be recovered correctly. 
Only a {\em fraction} of vertices, i.e. $(1-\epsilon)n$ for some $\epsilon >0\,$, can be guaranteed to be recovered correctly. 
  For example, in the case
  of two symmetric communities, this fraction should be greater than $1/2$ (which
  one can obtain just by random guessing).
  
%\aminr{(In general, this happens in the case of sparse graphs with some $p_k$'s and $q$ close to each other.)}{a reference maybe?} 

%\roy{What's so special about $1/2$? (especially for multiple clusters)}
%\red{From AbbeSandon15:``This only requires the algorithm to output a partition of the nodes which is positively correlated with the true partition (whp). Note that weak recovery is relevant in the fully symmetric case where all nodes have identical average degree, since otherwise weak recovery can be trivially solved. If the model is perfectly symmetric, like the SBM with two equally-sized clusters having the same connectivity parameters, then weak recovery is non-trivial."}

\item {\em Detectability.} One may construct a partition of the
  graph which is correlated\footnote{In this context, this means doing
    better than guessing.} with the true %(or ground) 
    partition, but one cannot \emph{guarantee} any kind of quantitative improvement
  over random guessing. This happens in very sparse
  regimes when some $p_k$'s and $q$ are of the same, small, order; e.g. see \cite{mossel2014reconstruction}. 
%  \amin{probably one good reference for this definition is given in Theorem 2.4 of \cite{mossel2014reconstruction} where undetectability is presented as indistinguishability from a \ER model. }

  %\item {\bf impossibility regime.} In this regime, information-theoretic considerations make detectability impossible. (All $p_i$'s and $q$ are too small and too close together; the graph may be indistinguishable from $\mathcal{G}(n, q)$.)
\end{enumerate}

It may appear at first that the differences between exact recovery with strong and weak consistencies (the first two regimes above) are small; to illustrate the
  differences, consider the situation when one has a very large (sized
  $n$) social network with a particular set of nodes of interest, which
  may also be large but $o(n)$. 
  %Having a high-probability, 
  An exact 
  recovery algorithm with strong consistency guarantees that, with high probability, \emph{all} of
  the nodes of interest will be correctly labeled. 
  %Having a high-probability, 
  An exact recovery algorithm with weak consistency %weak recovery algorithm 
  can guarantee that \emph{any} of the nodes will
  be correctly labeled with high probability, but may yield absolutely
  no guarantees about the entire set (in fact, depending on the set
  size, the probability that some nodes will be mislabeled may be
  $O(1)$). 
% amin: we can add: if any node is recoverable w.h.p., then the probability of the set of interest, of size $o(n)$ being recoverable can go to zero. Also, 
 In other words, in such setting, while the probability of correct recovery for a {\em fixed} set of $n-o(n)$ vertices may be zero, the probability of correct recovery for {\em some} set of $n-o(n)$ vertices is close to one. %The latter is called exact recovery with weak consistency. 

%There is reason to believe that these regimes have sharp thresholds. 
%Under the lowest threshold, information-theoretic considerations make detectability of the partition impossible. 
%Recently, the topic of \emph{sharp thresholds} for the above regimes has seen an unprecedented burst of activity; 
%although thresholds have only been shown to exist for a few particular cases (for a complete characterization, see below), the understanding of the various regimes has progressed greatly in the last couple of years. 
%\amin{how about referring to some of the works here?}

\paragraph{Thresholds.} %\label{sec:thresholds_review}
Recently, there has been significant interest in determining \emph{sharp thresholds} (or phase transitions) for the various parameter regimes. 
Currently, the best understood case is the SBM with only two communities of equal size (which we refer to as binary SBM hereafter) for which all of the four regimes above have been identified and characterized in a series of recent papers \cite{coja2010graph,mossel2014reconstruction, mossel2013proof, massoulie_proof,mossel14belief, mossel2014consistency, abbe2014exact, hajek2014achieving}. Moreover, tractable algorithms have been proposed and they work down to the information-theoretical thresholds; i.e., information-theoretical and computational thresholds coincide for the case of binary SBM.  
%along with tractable reconstruction algorithms that work down to the thresholds, 
% \red{ the phase transition for exact recovery was only obtained last year for the case of two symmetric communities \cite{abbe2014exact,mossel2014consistency}}

Aside from this case, 
%the only other proven threshold was given in a recent paper by Abbe and Sandon \cite{abbe2015community}. 
Abbe and Sandon \cite{abbe2015community} proved the existence of an information-theoretic threshold for exact recovery in the case when the number $r$ of communities is fixed and all community sizes are $O(n)$ (while the connectivity probabilities $p_k,q$ are $O(\log n/n)$). 
%In particular, it is important to point out that above the threshold, exact recovery is efficient, with almost linear-time complexity. \red{add a description for \cite{abbe2015recovering}.}
%In particular, in \cite{abbe2015recovering}, they provided an almost linear-time algorithm that automatically learns the parameters of SBM and recovers the model, down to the information-theoretic threshold. 
In particular, in \cite{abbe2015community}, they provided an almost linear-time algorithm using the knowledge of model parameters that works down to this information-theoretic threshold. Such knowledge is shown to be unnecessary in a fully agnostic algorithm developed in \cite{abbe2015recovering}. 

%\amin{maybe we can squeeze this here: Abbe and Sandon \cite{abbe2015community} provided a necessary and sufficient condition for recoverability when the sizes of communities are linear in $n\,$. While their result confirm ours in Example \ref{ex:we-can1}, it is not applicable for other examples discussed above; the quantity of interest for recoverable configurations goes to zero for some pairs of communities while required by \cite{abbe2015community} to be at least equal to one. }

Outside of the settings described above, results tend to be inconclusive 
%partial \amin{misleading English word} 
where not all the regimes are well understood and the bounds incorporate large or unknown constants. 
%\maryam{alternatives to 'incomplete': results are inconclusive, many questions are unresolved, thresholds are unknown,  etc?} and consist of bounds which incorporate large or unknown constants. 
Although we do not aim to give an exhaustive review of the existing literature, we will mention the main state-of-the-art results for the regimes identified above.

\begin{enumerate}[\bf 1.]
\item \textbf{Exact Recovery (Strong Consistency).} Many partial results are available for general SBM,
  yielding upper bounds on the thresholds for efficient regimes, or
  lower bounds for exact recoverability; for example Chen and Xu \cite{chen2014statistical} which served as an inspiration for this paper. 
  %A good example, which has served as a source of inspiration for this paper, is the work of \cite{chen2014statistical}. 
  %coming on the heels of \cite{chen2012clustering}. \maryam{do we need to mention the earlier paper?}
  The results in \cite{chen2014statistical}  cover the regime when all clusters are
  equivalent, that is, all $p_k = p$ and there are $r$ clusters, each
  of size $K := n/r\,$; $r$ and $p$ are allowed to vary with
  $n\,$. Depending on $K$, $p$, $q$, and $n$, they characterize the conditions under which 
  1) exact recovery is {\em impossible}, 
  2) exact recovery is possible \emph{theoretically}, but not necessarily \emph{efficiently}, e.g., by the Maximum Likelihood Estimator, 
  3) exact recovery can be done efficiently, e.g., by a semidefinite programming relaxation of the ML estimator, 
  4) exact recovery can be done by a simple counting and thresholding procedure.

The bounds for these regimes in \cite{chen2014statistical} are not shown to be
sharp thresholds, but they work down to the limit of cluster connectivity for $p$ and $K$, which with $K = O(n^{\beta}$) for some constant $0 \leq \beta \leq 1$, results in $p = O(\log n/K)$ (further lowering of $p$ will result in a disconnected graph, and as such strong recovery becomes impossible.) % Their results are nicely summarized in Figure \ref{fig:4regimes-CX14}.
The downside of \cite{chen2014statistical} lies in the very strong assumption of equivalent clusters. The difficulty of such assumption in heterogeneous SBM wil be discussed in detail in Section \ref{sec:examples}. 
%For the case when $p_k, ~q = O(1)$, other examples of algorithms for efficient strong recovery (and which also consider the practical case when only partial observations are available) are the convex optimization approaches in \cite{oymak2011finding, VinayakOH14}.

%\roy{\cite{cai2014robust} offered a SDP based on convexified MLE by noting that the target recovery matrix is both low-rank and positive-semidefinite. One interesting feature in their proposed SDP is that, arbitrary outliers of relatively small size can be allowed to ensure exact recovery if we introduce another tuning parameter which penalizes the trace of the target recovery matrix. Theoretical analysis reveals that their SDP is comparable with best known results in the cases of \emph{balanced clusters} and \emph{equal probabilities}. However, the complexity of their model is still parametrized by the heuristic $p_{\min},n_{\min}$, which excludes many useful examples.}

\item \textbf{Almost Exact Recovery (Weak Consistency).} 
  This case has not been extensively treated in the literature. 
%  Other than in the example of \cite{} \amin{I don't know what this ref is} listed in the previous regime (due to the formulation of the problem), the most prominent recent work has been done by 
  Yun and Proutiere \cite{yun2014community} studied the case when there is a finite number of clusters, all of size $O(n)$, and such that all intra-cluster probabilities $p_k$ are equal to $p$. They find a characterizing condition for weakly consistent recovery in terms of $p$, $q$, and $n$; this condition was rediscovered in the case of the binary SBM by Mossel, Neeman and Sly \cite{mossel2014consistency}; for this latter case it can be stated as
  \beqa\label{eqn:weak_recovery}
   n \frac{(p-q)^2}{p + q} \rightarrow
  \infty~.
  \eeqa
% \red{  Roughly, this means that anything such that $np \rightarrow \infty$ will work for certain $q$ of the same order as $p$ or smaller. }
\cite{yun2014community} is the first to give a lower bound
  on the threshold. In their studied case this lower bound coincides with the
  upper bound, which they show 
  %that the upper bound is the same as the lower bound 
   by providing a spectral algorithm (based on an algorithm by
  Coja-Oghlan \cite{coja2010graph}) with a simpler analysis.

Previous to their results, there have been other
  methods/algorithms to show the possibility of weakly consistent recovery;
  although the algorithms used may be even simpler (e.g., 
  Rohe, Chatterjee, Yu \cite{rohe2011spectral}, which is spectral), they generally do
  not come close to the threshold. %(in the case of \cite{} \amin{I don't know what this ref is}, it only covers very dense graph regimes, many/most of which have now been identified as falling under the strong recovery regime). 

Previously, weakly consistent recovery has been studied by Rohe, Chatterjee, Yu \cite{rohe2011spectral} using a spectral algorithm (based on an algorithm by Coja-Oghlan \cite{coja2010graph} with a simpler analysis), but the results do not come close to the threshold where $p,q$ is required to be almost $O(1)$(up to logarithmic factors). 

% DOESN'T BELONG HERE: 
% Finally, Yun and Proutiere present two kinds of algorithms; non-adaptive (when the edges that are revealed are pre-selected) and adaptive (when the algorithm can adaptively choose a (bounded) set of edges to explore). The first type of algorithms are similar to the ones studied here, though they are more general (we consider in this paper the case when either we have access to the entire adjacency matrix, or to a set of uniformly sampled entries). 

Recently, Zhang and Zhou \cite{zhangminimax15} obtained similar result as (\ref{eqn:weak_recovery}) under approximately same-sized communities, with the smallest inter-cluster connectivity parameter $p$ and the highest intra-cluster connectivity parameter $q\,$, by adopting a minimax approach. They show that weak recovery is possible if 
	\beqas
	\frac{n(p-q)^2}{pK\log K}\to \infty,
	\eeqas
and is impossible if
	\beqas
	\frac{n(p-q)^2}{pK}=O(1)
	\eeqas
where $K$ is the number of clusters which is allowed to grow. %These together constitute the minimax boundary over the parameter space $(n,p,q,K,\beta)$ where $\beta$ characterizes the extent to which the cluster are same sized. They also derive an explicit penalized ML estimator achieving the minimax rates.-Roy }
Later, \cite{gao2015achieving} proposed a computationally feasible algorithm that provably achieves the optimal misclassification proportion given above. 

\item \textbf{Partial Recovery.} 
%Already mentioned above, Coja-Oghlan's work \cite{coja2010graph} is a great example which established a benchmark. 
%Already mentioned above, Coja-Oghlan's work \cite{coja2010graph} gave the first result on partial recovery in the bounded degree regime. found asymptotic conditions on these average degrees that guarantee partial recovery. 
 Coja-Oghlan \cite{coja2010graph} extended the asymptotic analysis of SBM to bounded degree regimes and was the first to give \emph{partial recovery} results. For the binary SBM case, his conditions amount roughly
  to the following: for $p=a/n$ and $q = b/n$ for some constants $a,b$, there exists
  some large constant $C$ such that, if $(a-b)^2 \geq C (a+b)
  \log(a+b)$, then partial recovery is possible, and the fraction of
  recovered vertices is upper bounded by a function of $C$. 
%  \aminr{(later, Mossel, Neeman and Sly {mossel14belief} showed that the belief propagation algorithm achieves the bound, in the case of the binary SBM).}{--- OLD result see their v3}
%\red{results of Coja-Oghlan \cite{coja2010graph} imply that a particular spectral method applied to the graph's adjacency matrix achieves non-trivial reconstruction, but this applies only when above the conjectured threshold by a possibly large constant.}
Following \cite{coja2010graph}, a series of works by \cite{decelle2011asymptotic, mossel2014reconstruction, massoulie_proof, mossel2013proof} established a sharp threshold for {\em detection} in binary SBM. 
%First, Coja-Oghlan .... 
Decelle et al \cite{decelle2011asymptotic} conjectured a sharp threshold at $(a-b)^2=2(a+b)\,$, based on non-rigorous ideas from statistical physics. 
Later, \cite{mossel2014reconstruction} showed %the negative part, 
that below this threshold it is impossible to cluster, or even to estimate the model parameters from the graph.  
Finally, \cite{massoulie_proof,mossel2013proof} provided an algorithm which efficiently outputs a labeling that is correlated with the true community assignment when $(a-b)^2 >2(a+b)\,$. 
Mossel, Neeman and Sly \cite{mossel14belief} proposed an algorithm using a variant of belief propagation that is optimal in the sense that if $(a-b)^2> C(a + b)$ for some constant $C$ then the algorithm achieves the optimal fraction of nodes labelled correctly. 
\dgreen{
}

For the general SBM in the bounded average degree regime, recently, Guedon and Vershynin \cite{guedon2014community} analyzed a convex optimization based approach, and Le, Levina, and Vershynin \cite{le2015sparse} analyzed a simple spectral algorithm, achieving similar upper bounds on the threshold of partial recovery. 
The proofs make use of the Grothendieck inequality. 
\cite{guedon2014community} offers a convex optimization approach for obtaining a correct
labeling of a $(1-\epsilon)$ fraction of the vertices for arbitrarily small $\epsilon$. 
%the upper bounds they obtain for the density of the graph depends on the difference between the smallest inter-cluster connectivity parameter $p = a/n$, the highest intra-cluster connectivity parameter $q = b/n$, and the average variance of the edges $\bar{p} = g/n$, \maryam{unclear what $g$ is, may delete?} as well as $\epsilon$, and it bears resemblance to the Coja-Oghlan condition. 
%The probability of failure is exponentially small in $n$. 
The particular formulation of the convex problem is not
crucial and can be changed without significant change to the bound
itself. However, it is unclear how their results evolve when the networks have
unbounded average degrees. 

%\red{can we shorten or remove this paragraph?: }
%Le, Levina, and Vershynin \cite{le2015sparse} use a spectral method (also using the Grothendieck inequality) that involves the combinatorial Laplacian of the network. Since these Laplacians do not concentrate in the case of bounded average-degree networks, the authors came up with a regularization trick that overcomes this obstacle: adding a small quantity $\tau$ to the matrix entries (such that $n \tau$ is the empirical average degree) creates ``weak'' edges between any two vertices 
%%(and also creating ``weak'' loops) and thus
%and increases the total  degree by $n \tau$. A similar trick was adopted in \cite{qin2013regularized}. \red{This method works for a finite number of communities. }
%As a result, the eigenvalues of the Laplacian of the regularized graph concentrate. 
%%and working with the resulting matrix Laplacian instead of the normal one, the eigenvalues of the Laplacian concentrate. 
%%As a result of concentrating the Laplacian, 
%Assuming a finite number $r$ of clusters, the first $r$ leading eigenvectors 
%%of the Laplacian of the shifted matrix 
%can be used in conjunction with a $r$-means algorithm to approximately recover the labels. 
%%
%Le, Levina, and Vershynin \cite{le2015sparse} proposed a spectral method, for the bounded average degree regime as above, with a degree correction step without which the graph Laplacian does not concentrate.
Le, Levina, and Vershynin \cite{le2015sparse} proposed a spectral method with degree correction when the average degree regime of the network is bounded. As a result of the degree correction, the graph Laplacian concentrates (which otherwise does not, in the bounded average degree regime) and hence the leading eigenvectors of the Laplacian can be used to approximately recover the labels. 
A similar degree correction trick was adopted in \cite{qin2013regularized}. %\red{This method works for a finite number of communities. }
%\red{RCY11 works for growing number of clusters: }
It should be noted that in \cite{rohe2011spectral}, the authors used the fact that although the Laplacian does not concentrate, the square of the Laplacian does, and obtained good partial solutions in a much denser regime (smallest degree being $O(n/\log n)$).

%\red{AS15: ``For the symmetric case (more than 2 communities), the information-theoretic and computational thresholds for weak-recovery remain open for more than 2 communities."}

\item \textbf{Detectability/Impossibility.} 
%The only results to mention here are for the binary SBM with $p=a/n$ and $q=b/n$ for which the threshold of detectability has been shown to be $(a-b)^2 < 2(a+b)$, and for the symmetric SBM with $r$ equivalent communities where the strongly empirically-supported conjecture of Decelle et al \cite{decelle2011asymptotic} states that  \aminr{when $(p-q)^2< rn (p+q)$}{when $(a-b)^2<c(r) (a+(r-1)b)$ for some $c(r)\leq r$},  the model is indistinguishable from a general \ER model; e.g. see Conjecture 7.2 in \cite{mossel2014reconstruction} for details. 
%
%Virtually nothing is known about this threshold in the general case, partly because the interest of the community has been toward more positive results. 
%
{
As mentioned above, for the binary SBM with $p=a/n$ and $q=b/n\,$, Decelle et al \cite{decelle2011asymptotic} conjectured that if $(a-b)^2 < 2(a + b)$ 
%then the node labels cannot be inferred from the unlabelled graph with better than 50\% accuracy (which could also be done just by random guessing). 
one cannot infer the community assignments with better than 50\% accuracy which can be achieved by random guessing. 
The conjecture was later verified by \cite{mossel2014reconstruction} as pointed out above.  
%they showed that it is impossible to cluster, or even to estimate the model parameters from the graph when $(a-b)^2 < 2(a+b)\,$. 
For the symmetric SBM with $r$ equivalent communities (of the same size and connection probabilities), the strongly empirically-supported conjecture of Decelle et al \cite{decelle2011asymptotic} states that 
  %\aminr{when $(p-q)^2< rn (p+q)$}
  {when $(a-b)^2<c(r) (a+(r-1)b)$ for some $c(r)\leq r$}, 
  the model is indistinguishable from a general \ER model; e.g. see Conjecture 7.2 in \cite{mossel2014reconstruction} for details. 
}
\end{enumerate}

As mentioned in the beginning of this section, it has been proven that there is no gap between the information-theoretic and computational thresholds for binary SBM. On the other hand, while the information-theoretic threshold for partial recovery of more than 2 communities is still unknown, \cite{mossel2014reconstruction} conjectured a gap exists for partial recovery for more than 4 communities. Similarly, sharp thresholds for exact recovery of multiple communities are still unknown (see \cite{abbe2015community}).

In addition to the papers mentioned above, the interested reader will find good surveys of current literature in \cite{chen2014statistical,abbe2015community, amini2014semidefinite, mossel14belief, mossel2014consistency}.
%; in addition to the survey, this paper provides some of the best existing bounds for exact recovery when $r$ is allowed to grow with $n$, provided that the size $n/r$ and connectivity parameter $p_i=p$ are the same for all communities. 

%\begin{table}[htbp]
%\footnotesize\centering\ara{1.3}%\vspace{.2in}
%\begin{tabular}{@{}l  c l l@{}}\hrulet
%model $\mathcal{M}$
%	&  $\tau_{\mathcal{M}}$
%	& quality of inference
%	& reconstruction methods
%	\\ \hrulet
%%------------------------------------------------------------------------------------
%binary SBM, 
%	& $\tfrac{\sqrt{a}-\sqrt{b}}{\sqrt{2}}$
%	& sharp exact recovery threshold 
%	& SDP \cite{hajek2014achieving} \\
%~~~$p=a\log(n)/n, q=b\log(n)/n$	
%	&
%	& ~~~at $\sqrt{a}-\sqrt{b}=\sqrt{2}$ \cite{abbe2014exact, mossel2014consistency} 
%	& method 2 \\\hline
%%------------------------------------------------------------------------------------
%binary SBM, 
%	& $\tfrac{(a-b)^2}{2(a+b)}$
%	& sharp detection threshold 
%	& \cite{massoulie_proof}\\
%~~~$p=a/n, q=b/n$	
%	&
%	& ~~~at $(a-b)^2=2(a+b)$ & \cite{mossel2013proof}\\
%	&&& belief propagation\cite{mossel14belief}\\\hline
%%------------------------------------------------------------------------------------	
%&&\\ \hrulet
%\end{tabular}
%\caption{}
%\label{tab:thresholds}
%\end{table}%

\subsection{This paper}  \label{sec:this-paper}

In this paper we study the general setup presented in Section \ref{sec:GSBM-def}, where the communities are not constrained to have the same size and connection probabilities, and where $r$ is allowed to grow with $n$. Our work is {\em concerned with exact recovery} and is based on \cite{chen2014statistical}. We provide the following:
\begin{itemize}
\item An information-theoretic lower bound, describing an impossibility regime (Theorem \ref{thm:impossibility}), 
\item An upper bound, describing a potentially ``hard" regime in which recovery is always possible,
  though not necessarily in an efficient way (Theorem \ref{thm:hard-recovery}). Here we assume the sizes of the communities $n_k\,$, for $k=1,\ldots,r\,$, are known. 

\item An upper bound for efficient recovery via a convex optimization algorithm similar to the one in \cite{chen2014statistical}, describing an ``easy" regime (Theorems \ref{thm:convex_recovery} and \ref{thm:convex_recovery2}). Here we assume the quantity $\sum_k n_k^2$ is known. 

%\red{mention: we don't have a converse theorem to characterize the boundary of the space for which convex recovery is possible. However, we explore the space of all possible configurations using two separate theorems which covers a nice variety of instances; see Examples \ref{ex:cvx-thm1-sqrtlogn} and \ref{ex:cvx-thm2-logn}.  }

\item A bound characterizing a very simple and efficiently solvable
  thresholding algorithm, if model parameters $p_k,q$ are known (Theorem \ref{thm:simple-recovery}). 
%\roy{We should probably be honest about this: Simple thresholding succeeds only if we have a priori information of model parameters $p_i,q$.}

\item Extensions of the above bounds to the case of partial
  observations, i.e., when each entry of the matrix is observed uniformly with
  some probability $\gamma$ and the results recorded.
%  \amin{TO BE MERGED HERE: The second important aspect is the inclusion of partial observation. We observe that if an estimator of the model parameter only takes the observed part of the adjacency matrix as its input, then we can think of the probabilistic model as the revised one by simply scaling model probabilities with observation ratio. See Section \ref{sec:partial} for more details. }
\end{itemize}

%The first four of the bounds mentioned above are illustrated in the case of two clusters in the picture below. 
%\begin{figure}[htbp]
%\begin{center}
%\begin{tikzpicture}
%  \draw[thick,->] (-.3,0) -- (4.2,0) node[below] {$p_1n_1$};
%  \draw[thick,->] (0,-.3) -- (0,4.2) node[above] {$p_2n_2$};
%
%  \draw[thick,domain=.3:3.8,smooth,variable=\x,dashed, gray!70] plot ({\x},{.3}) node[right] {impossible};
%  \draw[thick,domain=.3:3.8,smooth,variable=\y,dashed, gray!70]  plot ({.3},{\y});
%
%  \draw[thick,domain=.8:3.8,smooth,variable=\x,red] plot ({\x},{.8}) node[above] {recoverable};
%  \draw[thick,domain=.8:3.8,smooth,variable=\y,red]  plot ({.8},{\y});
%  
%  \draw[thick,scale=.15,domain=9:11.95,smooth,variable=\y,blue!70] plot ({(\y-2*sqrt(\y))^2},{\y}) node[above] {convex recovery};
%  \draw[thick,scale=.15,domain=9:11.95,smooth,variable=\y,blue!70] plot ({\y},{(\y-2*sqrt(\y))^2});
%
%  \draw[thick,scale= 1.6, domain=1:1.545,smooth,variable=\y,green] plot ({(\y)^2},{\y}) node[above] {simple thresholding};
%  \draw[thick,scale= 1.6, domain=1:1.545,smooth,variable=\y,green] plot ({\y},{(\y)^2});
%
%\end{tikzpicture}
%\caption{Regions of parameter space for GSBM corresponding to different computability regimes depicted for parameters of two of the clusters. The schematic is derived by simplifying the theoretical bounds presented in Theorems \ref{thm:impossibility}, \ref{thm:hard-recovery}, \ref{thm:convex_recovery} and \ref{thm:simple-recovery}.}
%\label{fig:4regimes-GSBM}
%\end{center}
%\end{figure}

Our setup is general and allows for any mix of clusters of all magnitudes and
densities. We illustrate the importance of considering such a model, as opposed to using summary statistics such as $n_{\min}$ and $p_{\min}\,$, by some examples later in this section. %in Section \ref{sec:this-paper}. 
This setup allowed us to identify the crucial quantities 
\beqa\label{def:relative_density}
\rho_k = n_k (p_k - q)\,,\quad \chidiv(p_k,q) = \frac{(p_k-q)^2}{q(1-q)}  \,,\quad \chidiv(q,p_k) = \frac{(p_k-q)^2}{p_k(1-p_k)}, 
\eeqa
%\red{$\rho$ will not be used until Main Results, only one time in pmin nmin which should be removed anyways } 
where $\rho_k$ is called the \emph{relative cluster density} for a cluster $k\,$, and $\chidiv$ represents the Chi-square divergence between two Bernoulli random variables with the given probabilities. We elaborate on these quantities in the beginning of Section \ref{sec:main-results}. 
The bounds resulting from our inequalities
bear resemblance to, and appear to be
generalizations of McSherry's \cite{mcsherry2001spectral}, allowing for the different
$n_k$'s and $p_k$'s. It is worth mentioning that we have explored the possibility of
allowing for a whole matrix of inter- and intra-cluster connectivity
probabilities (in other words, we looked at the case when instead of a
uniform probability $q$ of inter-cluster connection, we have different
connectivity probabilities $q_{kl}$ for each pair of clusters $(k,l)$, for $k
\neq l$.) The calculations can be followed through but at the cost of added notation complexity, with no clear shortcut, which we decided not to pursue. 
%but the added complexity of notation in the inequalities may render the bounds unusable in practice. 

%\subsection{Understanding the Results} 
Our results %are dependent on many parameters\maryam{omit this, unless we can say something interesting about the 'many parameters'}, and 
cover a wider set of cases than present in the existing literature. 
We give illustrative examples in Section \ref{sec:examples} to show that the setup we
consider and the results we obtain represent a clear
improvement over previous work. 
%For this purpose, we are providing below a set of 
%The examples illustrate important points, sometimes comparing and contrasting
%them with existing literature, and at other times emphasizing how our results, 
%Theorems \ref{thm:convex_recovery}, \ref{thm:convex_recovery2} and \ref{thm:hard-recovery}, complement each other. 
The examples emphasize how  
Theorems \ref{thm:convex_recovery}, \ref{thm:convex_recovery2} and \ref{thm:hard-recovery} (given in Section \ref{sec:main-results} with proofs and more details given in Appendices \ref{app:proof-convex}, \ref{app:proof-rec}), complement each other, 
and how they compare and contrast with existing literature. 
More details and justification for the claims made in the examples are given in Appendix \ref{app:verification}. 
%Finally, Section \ref{sec:main-results} contains our results, and proofs are given in Appendices \ref{app:proof-convex}, \ref{app:proof-rec}. 
%The verifications of the claims made throughout the examples are listed in Appendix \ref{app:verification}. 

\subsection{Examples}	\label{sec:examples}
In the following, a {\em configuration} is a list of cluster sizes $n_k$, their connection probabilities $p_k$, and the inter-cluster connection probability $q\,$.  A triple $(m,p,k)$ indicates $k$ clusters of size
$m$ each, with connectivity parameter $p\,$. We do not
worry about whether $m$ and $k$ are always integers; if they are not,
one can always round up or down as needed so that the total number of
vertices is $n$, without changing the asymptotics. Moreover, when the $O(\;)$ notation is used, we mean that appropriate constants can be determined. 

\newcommand{\ck}{$\checkmark$}
\newcommand{\no}{$\times$}
\begin{table}[htbp]
\begin{center}
\small
\begin{tabular}{l | l | ccc}
	&  & convex recovery & convex recovery & recoverability  \\
	& importance & by Thm.~\ref{thm:convex_recovery}& by Thm.~\ref{thm:convex_recovery2} 	& by Thm.~\ref{thm:hard-recovery}  \\[3pt] \hline
Ex.~\ref{ex:we-can1} 		&counter-example for $(p_{\min},n_{\min})$ & \no& \no&\ck \\
Ex.~\ref{ex:we-can2} 		&counter-example for $(p_{\min},n_{\min})$ &\ck&\ck&\ck \\
Ex.~\ref{ex:cvx-thm1-sqrtlogn} 	&$n_{\min}=\sqrt{\log n}$ &\ck&\no&\no \\
Ex.~\ref{ex:cvx-thm2-slogn} 	&$n_{\max}=O(n)$, many small clusters 
						&\ck %if \eqref{eq-condn-ex:cvx-thm2-slogn}
					        	&\ck  %if \eqref{eq-condn-ex:cvx-thm2-slogn}
					        	&\ck \\
Ex.~\ref{ex:cvx-thm2-logn} 	&$n_{\min}=O(\log n)$, spread in sizes &\no&\ck&\ck \\
%Ex.~\ref{ex:linear-config} 	& &&&& \\
Ex.~\ref{ex:hard} 	& small $p_{\min}-q\,$, all $p_k,q$ are $O(1)$ 
				&\ck%\red{\no} 
				&\ck 
				&\ck \\
\end{tabular}
\end{center}
\caption{A summary of examples in Section \ref{sec:this-paper}. Each row gives the important aspect of the corresponding example as well as whether, under appropriate regimes of parameters, it would satisfy the conditions of the theorems proved in this paper.}
\label{tab:examples}
\end{table}%

\subsubsection{Counter-examples for the $(p_{\min},n_{\min})$ heuristic} \label{sec:counter-ex}
% \paragraph{Non-Summarizable Heterogenous Examples.}
%\paragraph{The $p_{\rm{min}}$, $n_{\rm{min}}$ heuristic and why studying the general, heterogeneous case matters. } 
%  In \cite{chen2012clustering,chen2014statistical}, as well as elsewhere in the literature, there is an implication that the strong equivalent clusters assumption is not as restrictive as it might at first appear, as results can be extended to a more heterogeneous case by replacing $p$ with $p_{\min}$ and the cluster sizes with $n_{\min}\,$. 
In a heterogenous setup, one might think one can plug in $(p_{\min},n_{\min})$ in the results for homogenous SBM to identify recoverability regimes. 
While this simplistic approach will indeed yield some upper bounds on some of the ``positive" thresholds (i.e. if you can solve it for the simplistic case, you can also solve it for the more complex one, %according to BM-ordering \maryam{what's this? how does it fit here?}  \cite{amini2014semidefinite}), 
\emph{it can completely fail to correctly identify solvable subregimes}. %in understanding the recoverability situation.
The first two examples show why such a heuristic used for generalization attempts in the literature is not useful enough.  

%\amin{However, as it is illustrated in the following two examples, such an approach is not helpful  }
%\amin{this latter part, replacing $n_k$'s with $n_{\min}$ is not what people do!!! maybe we should say: } 
% \amin{In some other works, e.g. \cite{cai2014robust,zhangminimax15}, the results are in terms of $p_{\min}$ and $n_{\min}$ which can be very suboptimal in the heterogenous setup.}
%\roy{I think we may need to think more about how to explain `hard' and `simple' region. On one hand, `hand' region is the one that guarantees recoverability which may even be non-constructive. On the other hand, the discussion of `easy' region seems to be too narrowed to our particular proposed algorithm. It might be a good idea to fit these discussions in the `Main contribution' section.}

\begin{example}	\label{ex:we-can1}
Suppose we have two clusters of sizes $n_1 = n -\sqrt{n}$, $n_2 = \sqrt{n}$,
  with $p_1 = n^{-2/3}$ and $p_2 = 1/\log n$ while $q =
  n^{-2/3-0.01}\,$. As we will see, the bounds we obtain here in
  Theorem \ref{thm:hard-recovery} make it clear that this case is theoretically
  solvable (in the \emph{hard} regime). 
By contrast, Theorem 3.1 in \cite{cai2014robust} (specialized for the case of no outliers), requiring
\begin{equation}\label{eq:CL14-pmin-nmin}
n_{\min}^2(p_{\min}-q)^2 \gtrsim (\sqrt{p_{\min}n_{\min}} + \sqrt{nq})^2\log n \,,
\end{equation}
would fail and provide no guarantee for recoverability. 
%Similarly, Theorem 2.1 in \cite{chen2014statistical} would place the simpler configuration which uses $(p_{\min},n_{\min})$ in the impossible regime. \amin{using $n_{\min}=\sqrt{n}$ in homogenous setup requires assuming $r=\sqrt{n}$ communities: which is sort of incomparable to our case of 2 communities! - again: I don't think people take $n_{\min}$, while they do take $p_{\min}$.}\maryam{let's discuss this.}
\end{example}

\begin{example} 	\label{ex:we-can2}
Consider a configuration as
\beqas
\tri{n - n^{2/3}}{n^{-1/3+ \epsilon}}{1} \;\;,\;\;
\tri{\sqrt{n}}{O(\tfrac{1}{\log n})}{n^{1/6}} \;\;,\;\;
q = n^{-2/3+ 3 \epsilon},
\eeqas
where $\epsilon$ is some small quantity, e.g. $\epsilon = 0.1\,$. 
Either of Theorems \ref{thm:convex_recovery} and \ref{thm:convex_recovery2} verify that this case is in the \emph{easy} regime, and the partition can be recovered efficiently by solving a convex program, with high probability. By
contrast, using the $p_{\min} = n^{-1/3+\epsilon}$ and $n_{\min}=
\sqrt{n}$ heuristic, neither the condition of Theorem 3.1 in \cite{cai2014robust} (given in \eqref{eq:CL14-pmin-nmin}) nor the condition of Theorem 2.5 in \cite{chen2014statistical} is fulfilled, and thus we have no means
of reaching the same conclusion based on the $(p_{\min},n_{\min})$ heuristic.
\end{example}

%One may argue that the huge difference in the cluster sizes in the two
%examples above may indicate that we could use other means to overcome this issue. %Fixes have indeed been proposed (like ``peeling" the larger clusters off, one at a time, \cite{ailon2013breaking}), but the only guarantees are obtained when the graphs are dense ($p - q = O(1)$), which is not the case here. 

\subsubsection{Cluster sizes: small, large, and in-between} 
    The next three examples attempt to provide an idea of how wide the spread of cluster sizes can be, as characterized by our results. 
%Many (perhaps most) 
Most algorithms for clustering the SBM run into the problem of small clusters
\cite{chen2012clustering,boppana1987eigenvalues,mcsherry2001spectral},
often because the models employed do not allow for enough
parameter variation to identify the key quantities involved. 
%\amin{the following sentences are handwavy; $(p_k,n_k)$ is the whole information! }
{The bounds we obtain in this paper indicate that the ``correct" parameters are not the pairs $(p_k,n_k)$, but rather the relative cluster densities $\rho_k  = (p_k-q) n_k$ (which are related to the ``effective densities'' appearing in \cite{VinayakOH14}). This allows us to significantly vary the sizes of the clusters, and still be able to obtain exact recovery, as long as the relative densities are large enough. }

\begin{example}[ (smallest cluster size for convex recovery)] \label{ex:cvx-thm1-sqrtlogn}
%\amin{we can go with either of these two:}\bit\item 
Consider a configuration as 
\beqas
\tri{\sqrt{\log n}}{O(1)}{m} \;\;,\;\; 
\tri{n_2 }{O(\tfrac{\log n}{\sqrt{n}})}{\sqrt{n} } \;\;,\;\;
q = O(\tfrac{\log n}{n}),
\eeqas
where $n_2 = \sqrt{n} - m \sqrt{\log n / n}$ to ensure a total of $n$ vertices. Here, we assume $m\leq n/(2\sqrt{\log n})$ which implies $n_2 \geq \sqrt{n}/2\,$. 
%Here $m$ is a fixed number. 
It is straightforward to verify the conditions of Theorem \ref{thm:convex_recovery}. 
Notice that, in verifying the first condition for the second group of
clusters (with $p_2 = O(\tfrac{\log n}{\sqrt{n}})$), we need $p_2n_2 \gtrsim \log n_2$, which is
satisfied when $m$ is a constant. 

%\red{be satisfied as long as $m\leq n/(2\sqrt{\log n})\,$; in fact $m$ is fixed. ; hence we can have many (a fraction of $n$) tiny clusters (of size $\sqrt{\log n}$). Moreover, with the assumption on $m\,$, the right hand side of last condition is $O(\log n)\,$, and the condition is satisfied for both groups of clusters.  }

%\item Consider a configuration as 
%\beqas
%\tri{\sqrt{\log n}}{O(1)}{m} \;\;,\;\; 
%\tri{\sqrt{n}}{O(\tfrac{\log n}{\sqrt{n}})}{m' } \;\;,\;\;
%q = O(\tfrac{\log n}{n}) \,.
%\eeqas
%where $m' = \sqrt{n} - m \sqrt{\log n / n}$ to ensure a total of $n$ vertices. Here, we assume $m\leq n/(2\sqrt{\log n})$ which implies $m' \geq \sqrt{n}/2\,$. 
%%Here $m$ is a fixed number. 
%It is straightforward to verify the conditions of Theorem \ref{thm:convex_recovery}. 
%%Notice that, in verifying the first condition for the second group of clusters, we have to make sure that $p_2n_2 \gtrsim \log n_2$, which is satisfied when $m$ is a constant. 
%\eit

There are two important things to note in this example. First, to our knowledge, \emph{this is the first example in the literature for which SDP-based recovery works and allows the recovery of (a few) clusters of size smaller than $\log n$.} Previously, $\log n$ was considered to be the standard bound on the cluster size for exact recovery, as illustrated by Theorem 2.5 of \cite{chen2014statistical} in the case of equivalent clusters. 
We have thus shown that it is possible, in the right
circumstances (when sizes are spread and the smaller the cluster, the
denser), to recover very small clusters (up to $\sqrt{\log n}$ size),
\emph{if there are just a few of them (at most polylogarithmic in $n$).} 
The significant improvement we made in the bound on the size of the smallest
cluster is due to the fact that we were able to perform a closer
analysis of the SDP machinery (which we provide in the proof of Theorem
\ref{thm:convex_recovery}). For more details, see Section \ref{app:proof-convex_recovery}. %, in particular Remark \ref{}. \amin{I don't know which remark this is.}
%, and it is possible that even more could be obtained
%with sharper tools. \red{I don't understand the point of this heuristic discussion.-Roy}

Secondly, %and related, 
the condition of Theorem \ref{thm:hard-recovery} is {\em not} satisfied. This is not an inconsistency (as Theorem \ref{thm:hard-recovery} gives only an upper bound for the threshold), but indicates the limitation of this theorem in characterizing all recoverable cases. 
\end{example}

%If all the clusters are of equal size, Theorem 2.5 of \cite{chen2014statistical} guarantees recovery when the cluster size is at least $\log n\,$. However, when the cluster sizes are spread, Theorem \ref{thm:convex_recovery} allows the smallest cluster to be of size $\sqrt{\log n}\,$, which is discussed in Example \ref{ex:cvx-thm1-sqrtlogn}. 

\paragraph{Spreading the sizes.} 
The previous example allows us to go lower than the standard $\log n$
bound on the cluster size for exact recovery; however, we can solve
only if the number of very small clusters is finite. 
On the other hand, Theorem \ref{thm:convex_recovery2} provides us with
the option of having many small clusters but requires the smallest
cluster to be of size $O(\log n)\,$. Since the maximum cluster size is
$O(n)$, one may ask what kind of a spread can be achieved with the
help of Theorem \ref{thm:convex_recovery2}. 
In Example \ref{ex:cvx-thm2-slogn}, we assume a cluster of size $O(n)$ and examine how small $n_{\min}$ can be for Theorem \ref{thm:convex_recovery2} to guarantee exact recovery by the convex program. 
Similarly, in Example \ref{ex:cvx-thm2-logn}, we fix $n_{\min}=\log n$ and examine how large $n_{\max}$ can be. 

\begin{example} \label{ex:cvx-thm2-slogn}
Consider a configuration where small clusters are dense and we have a big cluster, 
\[
\tri{\tfrac{1}{2}n^\epsilon}{O(1)}{n^{1-\epsilon}} \;\;,\;\;
\tri{\tfrac{1}{2}n}{n^{-\alpha}\log n}{1} \;\;,\;\;
q = O(n^{-\beta}\log n), 
\]
%With $\epsilon>0$, the size condition of Theorem \ref{thm:convex_recovery} is satisfied. 
with $0<\epsilon<1$ and $0<\alpha< \beta<1$.
%(as we do not want the the inter-cluster or the intra-cluster probabilities too small) \maryam{we can delete this, doesn't say much, unless we say something more quantitive about "too small"}
Then the conditions of Theorems \ref{thm:convex_recovery} and \ref{thm:convex_recovery2} both require that
%the following to hold
\begin{align} \label{eq-condn-ex:cvx-thm2-slogn} 
%2\alpha+\epsilon <2 \;,\; \alpha+2\epsilon >1 \;,\; \beta>\max\{\alpha,\, 2\alpha-\epsilon \} \,.
%\max\{ \frac{1}{2}(1-\alpha), 2\alpha-\beta\} <\epsilon<2(1-\alpha).
\tfrac{1}{2}(1-\alpha) <\epsilon<2(1-\alpha) \quad,\quad  \epsilon> 2\alpha-\beta
\end{align}
%where $\alpha=\tfrac{1}{3}\,$, $\epsilon=\tfrac{2}{3}\,$, and any $\beta>\tfrac{1}{3}$ works. \red{BUT we'd like a small $\epsilon$}. for example, pick $\alpha = 1-2\epsilon$ and $\beta= 2\alpha\,$, for any $0<\epsilon < 1/4\,$. 
and are depicted in Figure \ref{fig:spread-size}. Since we have not specified the constants in our results, we only consider strict inequalities. 

\begin{figure}[htbp]
	\newcommand{\tickscl}{0.6}
	\newcommand{\textscl}{0.7}
\psfrag{right}[cc][cc][\textscl][-60]{$2\alpha+\epsilon=2$}
\psfrag{bottom}[cc][cc][\textscl][5]{$\alpha+2\epsilon=1$}
\psfrag{front}[tc][cc][\textscl][15]{$2\alpha=\beta+\epsilon$}
\psfrag{alpha}[lt][cb][\textscl]{~~~~~$\alpha$}
	\psfrag{0a}[cc][cc][\tickscl]{$0$}
	\psfrag{1a}[cc][cc][\tickscl]{$0.25$}
	\psfrag{2a}[cc][cc][\tickscl]{$0.5$}
	\psfrag{3a}[cc][cc][\tickscl]{$0.75$}
\psfrag{beta}[lt][cc][\textscl]{~~~~~$\beta$}
	\psfrag{0b}[cc][cc][\tickscl]{$0$}
	\psfrag{1b}[cc][cc][\tickscl]{$1/3$}
	\psfrag{2b}[cc][cc][\tickscl]{$2/3$}
	\psfrag{3b}[cc][cc][\tickscl]{$1$}
\psfrag{eps}[cb][cc][\textscl]{$\epsilon$}
	\psfrag{0e}[cc][cc][\tickscl]{$0$}
	\psfrag{1e}[cc][cc][\tickscl]{$0.2$}
	\psfrag{2e}[cc][cc][\tickscl]{$0.4$}
	\psfrag{3e}[cc][cc][\tickscl]{$0.6$}
	\psfrag{4e}[cc][cc][\tickscl]{$0.8$}
	\psfrag{5e}[cc][cc][\tickscl]{$1$}
\psfrag{A}[cc][cc][\textscl]{} % 0.6,1,.2
\psfrag{B}[cc][cc][\textscl]{} % .75,1,.5
\psfrag{C}[cc][cc][\textscl]{} % 1/3,1/3,1/3
\psfrag{D}[cc][cc][\textscl]{} % 2/3,2/3,2/3
\begin{center}
\vskip -1em
\includegraphics[scale=.5]{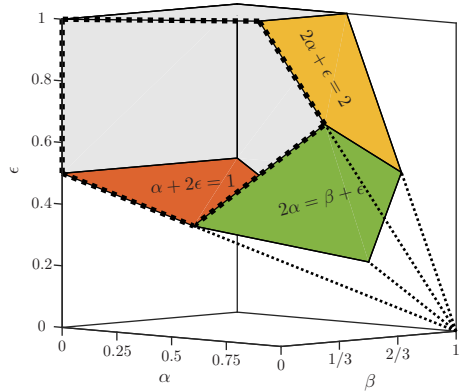}
\vskip -2em
\caption{The space of parameters in Equation \ref{eq-condn-ex:cvx-thm2-slogn}. The face defined by $\beta=\alpha$ is shown with dotted edges. The three gray faces correspond to $\beta=1\,$, $\alpha=0$ and $\epsilon=1\,$.  
The green plane (corresponding to the last condition in \eqref{eq-condn-ex:cvx-thm2-slogn}) comes from controlling the intra-cluster interactions uniformly (see \eqref{eq:Bkk-bound-before} and \eqref{eq:Bkk-bound}) which might be only an artifact of our proof and can be possibly improved. 
}
\label{fig:spread-size}
\end{center}
\end{figure}

Notice that the small clusters are as
dense as can be, but the large one is not necessarily very dense. By
picking $\epsilon$ to be just over $1/4$, we can make $\alpha$ just
shy of $1/2$, and $\beta$ very close to $1$. As far as we can tell,
there are no results in the literature surveyed that cover such a case, although the %ingenious (replaced with 'clever')
clever
``peeling'' strategy introduced in \cite{ailon2013breaking} would recover the largest cluster. 
The strongest result in \cite{ailon2013breaking} that seems applicable here
is Corollary 4 (which works for non-constant probabilities). The
\cite{ailon2013breaking} algorithm works to recover a large cluster
(larger than $O(\sqrt{n} \log^2n)$), subject to existence of a gap in the cluster sizes (roughly, there
should be no cluster sizes between $O(\sqrt{n})$ and $O(\sqrt{n}
\log^2n)$). Therefore, in this example, after a single iteration, the
algorithm will stop, despite the continued existence of a gap, as there
is no cluster with size above the gap. Hence the ``peeling''
strategy on this example would fail to recover all the clusters.

%By examining the proofs of Theorems \ref{thm:convex_recovery} and \ref{thm:convex_recovery2}, we observe that 

%The first two conditions in \eqref{eq-condn-ex:cvx-thm2-slogn} come from the concentration of the adjacency matrix around its expected value, while the last condition is a result of upper bounding the intra-cluster interactions (see \eqref{eq:Bkk-bound}). It is possible that the last condition, i.e., $\beta>2\alpha-\epsilon$, can be improved by using a more careful analysis, and such an improvement would allow for a smaller minimum community size (characterized by $\epsilon$). 

\end{example}

%\red{PREVIOUSLY:
%\[
%\tri{\log n}{O(1)}{\tfrac{n}{2\log n}} \;\;,\;\;
%\tri{\tfrac{1}{2}n}{n^{-1/2+\epsilon}\log n}{1} \;\;,\;\;
%q = \tfrac{\log n}{n} \,.
%\]
%{It seems that the first condition of Theorem \ref{thm:convex_recovery} fails for the small clusters; we have made a mistake today. Moreover, the last condition of both Theorems \ref{thm:convex_recovery} and \ref{thm:convex_recovery2} fail! }
%
%{verify: The simple counting algorithm in Section \ref{sec:simple} cannot handle this case. }
%}

\begin{example}	\label{ex:cvx-thm2-logn}
Consider a configuration with many small dense clusters. We are
interested to see how large the spread of cluster sizes can be for the
convex recovery approach to work. As required by Theorems \ref{thm:convex_recovery} and \ref{thm:convex_recovery2}
%and to maintain $\sigma_{\max}$ (defined in \eqref{def:sigmax}) at an achievable level \maryam{instead of 'maintain at an achievable level' (awkward sentence) it may be better to say 'to control $\sigma_{\max}$'}, 
and to control $\sigma_{\max}$ (defined in \eqref{def:sigmax})
the larger a cluster, the smaller its connectivity
probability should be; therefore we choose the
largest cluster at the threshold of connectivity (required for recovery). 
Consider the following cluster sizes and probabilities: 
\[
\tri{\log n}{O(1)}{\tfrac{n}{\log n}- m \sqrt{\tfrac{n}{\log n}}} \;\;,\;\;
\tri{\sqrt{n\log n}}{O(\sqrt{\tfrac{\log n}{n}})}{m} \;\;,\;\;
q = O(\tfrac{\log n}{n}),
\]
where $m$ is a constant. Again, we round up or down where necessary
to make sure the sizes are integers and the total number of vertices
is $n$. All the conditions of Theorem \ref{thm:convex_recovery2} are
satisfied, hence we conclude that exact convex recovery is
possible in this case. 

Note that the last condition of Theorem \ref{thm:convex_recovery} is not
satisfied since there are too many small clusters. Also note that alternate methods proposed in the literature
surveyed would not be applicable; in particular, the
gap condition in \cite{ailon2013breaking} is not satisfied for this case from the
start. 
\end{example}

%\textcolor{Purple}{Not sure what to do with Example \ref{ex:linear-config}.}
%\begin{example}	\label{ex:linear-config}
%Consider $p_s>p_e$ and $n_s<n_e$ and $r+1$ clusters of sizes $n_i = n_s + \tfrac{i}{r}(n_e-n_s)$ with respective probabilities $p_i = p_s + \tfrac{i}{r}(p_s-p_e)\,$, for $i=0,\cdots,r\,$; i.e. a configuration as 
%\[
%\tri{n_s}{p_s}{1} \;\;,\;\;\ldots \;\;,\;\;
%\tri{n_s + \tfrac{i}{r}(n_e-n_s)}{p_s + \tfrac{i}{r}(p_s-p_e)}{1} \;\;,\;\;\ldots \;\;,\;\;
%\tri{n_e}{p_e}{1}
%\]
%It is easy to show that 
%\[
%\min_i\, p_in_i = p_sn_s \wedge p_en_e \quad,\quad
%\max_i\, p_in_i \leq \frac{(p_sn_e - p_en_s)^2}{4(n_e-n_s)(p_s-p_e)} \,.
%\]
%Moreover, assume $q=\tfrac{1}{2}p_e\,$. Theorem \ref{thm:convex_recovery2} guarantees exact recovery via a convex program as long as 
%\[
%p_en_s \gtrsim \tfrac{\log^2 n}{n}\quad,\quad
%n_sp_s \wedge n_ep_e \gtrsim \frac{(p_sn_e - p_en_s)^2}{(n_e-n_s)(p_s-p_e)} \log n \,.
%\]
%\red{Power-law for sizes? and then $p$'s should change to meet the minimum density requirement}
%\end{example}

\subsubsection{Closeness of $p_{\min}$ and $q$}
Finally, the following examples illustrate how small $p_{\min}-q$ can be in order for the recovery, respectively, the convex recovery algorithms to still be guaranteed to work. Note that the difference in $p_{\min} - q$ for the two types of recovery is noticeable, indicating that there is a significant difference between what we know to be recoverable and what we can recover efficiently by our convex method. 
We consider both dense graphs (where $p_{\min}$ is $O(1)$) and sparse ones. 

\begin{example}	\label{ex:hard}
Consider a configuration where all of the probabilities are of $O(1)$ and
\beqas
\tri{n_1}{p_{\min}}{1} \;\;,\;\;
\tri{n_{\min}}{p_2}{1} \;\;,\;\;
\tri{n_3}{p_3}{\tfrac{n - n_1-n_{\min}}{n_3}} \;\;,\;\;
q = O(1),  
\eeqas
where $p_2-q$ and $p_3-q$ are $O(1)\,$. On the other hand, we assume $p_{\min}-q=f(n)$ is small. For recoverability by Theorem \ref{thm:hard-recovery},  we need $f(n) \gtrsim (\log n)/n_{\min}$ 
and $f^2(n) \gtrsim (\log n)/n_1\,$. Notice that, since $n\gtrsim n_1 \gtrsim n_{\min}\,$, we should have $f(n) \gtrsim \sqrt{{\log n}/{n}}\,$.  %\amin{if the two conditions in the line above are met, then we can use Theorem 2.3. The last one is only a necessary condition for those two and is not enough to guarantee that we can use 2.3. So we need to keep those two as the requirements.  }
%For the convex program to recover this configuration (by Theorem \ref{thm:convex_recovery}), we need $n_{\min}\gtrsim \sqrt{n}$ and $n_1\gtrsim \tfrac{1}{f(n)} \sqrt{n}\,$. Hence, for convex recovery when all probabilities are $O(1)$, \red{ we need $f(n) \gtrsim \frac{1}{\sqrt{n}}$.}
For the convex program to recover this configuration (by Theorem \ref{thm:convex_recovery} or \ref{thm:convex_recovery2}), we need $n_{\min}\gtrsim \sqrt{n}$ and $f^2(n) \gtrsim \max\{n/n_1^2\,,\, \log n/ n_{\min}\}\,$, while all the probabilities are $O(1)\,$.

%\red{from before: For $p_{\min}-q = n^{-\epsilon}$ and Theorem \ref{thm:convex_recovery}, (see Example \ref{ex:cvx-thm1-sqrtlogn}), we need $n_{\min}\gtrsim \sqrt{n} \vee n^{2\epsilon} \log n\,$. ???}
\end{example}

For a similar configuration to Example \ref{ex:hard}, where the probabilities are not $O(1)\,$, 
%but $p_2-q \gg p_{\min}-q=f(n)\,$. For recoverability by Theorem \ref{thm:hard-recovery}, we need $n_{\min} \gtrsim \tfrac{p_{\min}}{f(n)}(\log n)$ and $n_1 \gtrsim \tfrac{p_{\min}}{f^2(n)}(\log n)\,$; as the largest $n_{\min}$ and $n_1$ can be is $O(n)$, it follows that $f(n) \gtrsim \sqrt{\frac{p_{\min}\log n}{n}}$. 
recoverability by Theorem \ref{thm:hard-recovery} requires $f(n) \gtrsim \max\{\sqrt{p_{\min}(\log n)/n}\,,\, n^{-c}\}$ for some appropriate $c>0\,$.  % GRAPH?

Note that if all the probabilities, as well as $p_{\min}-q\,$, are $O(1)$, then by Theorem \ref{thm:hard-recovery} all clusters down to a logarithmic size should be recoverable. However, the success of convex recovery is guaranteed by Theorems \ref{thm:convex_recovery} and \ref{thm:convex_recovery2} when $n_{\min}\gtrsim \sqrt{n}\,$.

\section{Main Results}	\label{sec:main-results}
In this paper, we consider the heterogenous stochastic block model described in Section \ref{sec:GSBM-def}. Consider a partition of the $n$ nodes into $V_0,V_1,\ldots,V_r\,$, where $\abs{V_k} = n_k\,$, $k=0,1,\ldots,r\,$. Consider $\bar n = \sum_{k=1}^r n_k$ and denote  the number of isolated nodes by $n_0$ ; hence, $n_0+\bar n = n\,$. Ignoring $n_0\,$, we further define $n_{\min} = \min\{n_k:\; k=1,\ldots,r\}$ and $n_{\max} = \max\{n_k:\; k=1,\ldots,r\}\,$.
The nodes in $V_0$ are isolated and the nodes in $V_k$ form the community $\Cc_k = V_k\times V_k\,$, for $k=1,\ldots, r\,$. The union of communities is denoted by $\Cc = \cup_{k=1}^r \Cc_k$ and $\Cc^c$ denotes the complement; i.e. $\Cc^c = \{(i,j):\; (i,j)\not\in\Cc_k \text{ for any } k = 1,\ldots, r,\text{ and } i,j = 1,\ldots,n\}$. 
%True clusters are given as $\{\Cc_1,\ldots,\Cc_r\}$, where each cluster is of size $\abs{\Cc_k} = n_k$. 
Denote by $\Yc$ the set of admissible adjacency matrices according to a community assignment as above, i.e. 
\[
\Yc:=\{Y\in\{0,1\}^{n\times n}:\; Y \text{ is a valid clustering matrix over the partition } V_0,V_1,\ldots, V_r \text{ where }\abs{V_k}=n_k \}  \,.
\]
We will denote by $\mathbf{1}_C\in\mathbb{R}^{n\times n}$ a matrix which is 1 on $C\subset\{1,\ldots,n\}^2$ and zero elsewhere. $\log$ denotes the natural logarithm (base $e$), and the notation $\theta\gtrsim 1$ is equivalent to $\theta \geq O(1)\,$. A Bernoulli random variable with parameter $p$ is denoted by $\ber(p)\,$, and a Binomial random variable with parameters $n$ and $p$ is denoted by $\bin(n,p)\,$. 

Consider a distribution over random graphs with $V$ as their node set as defined in \eqref{eq:rand-graph-dist}. 
Each subset $V_k$ is endowed with an \ER graph structure $\mathcal{G}(n_k, p_k)$ for $k=1,\ldots,r\,$, and an edge is drawn between two nodes in different communities, independent of other edges, with probability $q\,$.  We assume that $p_k\geq q$ for $k=1,\ldots,r\,$. The goal is to recover the underlying clustering matrix $Y^\st$ exactly given a single graph drawn from this distribution. 
We will need the following definitions:
\bit
\item Define the {\em relative density of a community} as
\beqas
\rho_k=(p_k-q)n_k
\eeqas
which gives $\sum_{k=1}^r \rho_k = \sum_{k=1}^r p_k n_k - qn\,$. 

\item 
The Neyman Chi-square divergence (e.g., see \cite{cressie1984multinomial}) 
between two discrete random variables $\mu$ and $\pi$ (on the same support set of size $t$) is defined as 
\[
\fdiv{\chi_N^2}{\mu}{\pi}
= \sum_{i=1}^t \frac{\mu_i^2}{\pi_i}-1 
=  \sum_{i=1}^t \frac{(\mu_i-\pi_i)^2}{\pi_i} 
\]
%\maryam{what is $N$ in the subscript?} Neyman
and is always bounded below by the KL divergence; due to $\log x \leq x-1\,$. 
In the case of two Bernoulli random variables $\ber(p)$ and $\ber(q)\,$, the Neyman Chi-square divergence is given by
\begin{equation*}
\chidiv(p,q) := \frac{(p-q)^2}{q(1-q)}
\end{equation*}
and we have $\chidiv(p,q)\geq \kldiv(p,q):=\kldiv(\ber(p),\ber(q))\,$; see \eqref{eqn:KL-ineq}. 
Moreover, for $q<p\,$, when both $p$ and $q/p$ are bounded away from $1\,$, we have 
\begin{align}\label{eq:chidivqp-approx-p}
\chidiv(q,p) = p \frac{(1-q/p)^2}{1-p} \approx p \,.
\end{align}
Chi-square divergence is an instance of a more general family of divergence functions called $f$-divergences or Ali-Silvey distances \cite{ali1966general}. 
This family also has KL-divergence, total variation distance, Hellinger distance and Chernoff distance as special cases. Moreover, the divergence used in \cite{abbe2015community} is an $f$-divergence. %is a generalization of Hellinger and Chernoff distances.

\item %The average variance $\bar\sigma_k^2 = p_k(1-p_k)\,$, and 
Define the total variance $\sigma_k^2 = n_kp_k(1-p_k)$ over the $k$th community, and let $\sigma_0^2 = nq(1-q)\,$. Also, define
\begin{align}\label{def:sigmax}
\sigma_{\max}^2 = \max_{k=1,\ldots,r}\; \sigma_k^2 =  \max_{k=1,\ldots,r}\; n_kp_k(1-p_k)  \,.
\end{align}
\eit

\subsection{Convex Recovery}	\label{sec:convex}

%\todo \red{refer to SDP papers \cite{guedon2014community}\cite{amini2014semidefinite}\cite{}\cite{}. Mention that our results are made possible because of new probabilistic argument. }

%\aminr{While in the previous result characterizes a subset of exactly recoverable models}{...}, it does not provide us with any computational means for recovery. Next, 
We consider a convex optimization program for recovering the underlying clustering matrix $Y^\st = \sum_{k=1}^r \mathbf{1}_{\Cc_k}$ and characterize the models that are exactly recoverable using this program. In the following, $\norm{\cdot}_\st$ denotes the matrix nuclear norm or trace norm, i.e., the sum of singular values of the matrix. The dual to the nuclear norm is the spectral norm, denoted by $\norm{\cdot}\,$. 

\beqa \label{proc:convex-recovery}
%\begin{center}
\fbox{
\begin{minipage}[c][11.5em][c]{0.42\textwidth}{
{
%\refstepcounter{procedure}\label{proc:convex-recovery} \bf Procedure \theprocedure: 
\bf Convex Recovery
} 
\vskip.6em
\begin{algorithmic}%\label{alg:convexMLE}% [1]
\STATE {\bf input:} $\sum_{k=1}^r n_k^2$ \\[.3em]
\STATE {\bf output:} 
\beqas
\begin{array}{lll}
\hat Y = 
&\arg\underset{Y}{\max} 	& \sum A_{ij}Y_{ij} \\
&\mathrm{subject\; to}	& \norm{Y}_\st\leq \norm{Y^\st}_\st=n \\
&					& \sum_{i,j}Y_{ij}=\sum_{k}n_k^2 \\
&					& 0\leq Y_{ij}\leq 1
\end{array}
\eeqas
\end{algorithmic}
}\end{minipage}}
%\end{center}
\eeqa
%and that
%\bit
%\item either \cite{chatterjee2012matrix} $\min\{  p_i(1-  p_i)n_i/(\log n_i)^7,   q(1-  q)n/(\log n)^7\}\geq c_2$ and $r\exp(-C\min p_i(1-p_i)n_i, q(1-q)n)\leq c_3^{-1}n^{-c_3}$ for some constants $c_2,c_3$ and all $k$'s, where $C$ is a constant large enough.
%\item or \cite{tomozei2010distributed} ...
%\eit

We prove two theorems giving conditions under which the above convex program outputs the true clustering matrix with high probability. While the theorems are similar in terms of the methodology used, they differ in terms of the conditions we must impose. As we will see, Theorem \ref{thm:convex_recovery} allows us to describe a regime in which \emph{tiny} communities of size $O(\sqrt{\log n})$ are recoverable (provided that they are very dense and that only few tiny or small clusters exist; see Example \ref{ex:cvx-thm1-sqrtlogn}), while Theorem \ref{thm:convex_recovery2} covers a less restrictive regime in terms of cluster sizes, but allows us to recover clusters only down to $O(\log n)\,$; see Example \ref{ex:cvx-thm2-logn}. The proofs for both theorems along with auxiliary lemmas are given in Appendix \ref{app:proof-convex}. 

\begin{theorem}\label{thm:convex_recovery}
Under the heterogenous stochastic block model, 
%and $p_{\max}\leq 1-o(1)$ \maryam{do we need this?}
the output of convex recovery program in \eqref{proc:convex-recovery} coincides with $Y^\st\,$ with high probability, provided that
\beqas	%\label{cond:convex_recovery}
\rho_k^2  \gtrsim  \sigma_k^2 \log n_k \;\;,\;\;%\chidiv(q,p_k)  \gtrsim  \tfrac{\log n_k}{n_k} \;\;,\;\;
\chidiv(p_{\min},q) \gtrsim \tfrac{\log n_{\min}}{n_{\min}}  % it is \chidiv(p_i,q) , which can be observed to be minimum when plugging in the smallest p_i
\;\;,\;\;
\rho_{\min}^2 \gtrsim    \max\{\sigma_{\max}^2\,,\,nq(1-q)\,,\,  \log n  \} \;\;,\;\;
\sum_{k=1}^r n_k^{-\alpha} =  o(1)
\eeqas
for some $\alpha>0\,$, where $\sigma_k^2 = n_kp_k(1-p_k)\,$. 

\end{theorem}
The assumption $\sum_{i=1}^r n_i^{-\alpha} =  o(1)$ above is tantamount to saying that
  the number of small or tiny communities (where by tiny we mean communities of size $O(\sqrt{\log n})$) cannot be too large (e.g., the
  number of polylogarithmic-size communities cannot be a power of
  $n$). In other words, one needs to have mostly large communities
  (growing like $n^{\epsilon}$, for some $\epsilon>0$) for this assumption to be satisfied. Note, however, that the condition does \emph{not} restrict the number of clusters
  of size $n^{\epsilon}$ for any fixed $\epsilon>0\,$. 
The second theorem imposes more stringent conditions on the relative
density, %(the first condition of Theorem \ref{thm:convex_recovery2} can be expressed as $\rho_k\gtrsim \sigma_k \sqrt{\log n}$ while Theorem \ref{thm:convex_recovery} has $\rho_k\gtrsim \sigma_k \sqrt{\log n_k}$ as its first condition), 
but relaxes the condition that only a very small number of nodes can be in small clusters. 
\begin{theorem}\label{thm:convex_recovery2}
Under the heterogenous stochastic block model, 
%and $p_{\max}\leq 1-o(1)$ \maryam{do we need this?}
the output of convex recovery program in \eqref{proc:convex-recovery} coincides with $Y^\st\,$, with high probability, provided that
\beqas	%\label{cond:convex_recovery2}
\rho_k^2  \gtrsim  \sigma_k^2 \log n \;\;,\;\;%\chidiv(q,p_k)  \gtrsim  \tfrac{\log n}{n_k} \;\;,\;\;
\chidiv(p_{\min},q) \gtrsim \tfrac{\log n}{n_{\min}} 
 \;\;,\;\;
\rho_{\min}^2 \gtrsim   \max\{\sigma_{\max}^2\,,\, nq(1-q) \} \,.
\eeqas
\end{theorem}

\begin{remark}\label{rem:connected-convex} 
For exact recovery to be possible, we need all communities (but at most one) to be connected. Therefore, in each subgraph, which is generated by $\mathcal{G}(n_k,p_k)\,$, we need $p_k n_k > \log n_k$, for $k=1,\ldots,r\,$. 
Observe that this connectivity requirement is implicit in the first condition of Theorems \ref{thm:convex_recovery}, \ref{thm:convex_recovery2} which can be seen from \eqref{eq:chidivqp-approx-p}.
\end{remark}

%\red{(To the best of our knowledge, all of the existing convex recovery programs that involve nuclear norm have been analyzed by constructing a subgradient ... )}
Note that any convex optimization problem that involves the nuclear norm $\norm{Y}_\st$ (or equivalently, $\operatorname{tr}(Y)$ for $Y\succeq 0$) in its objective function or constraints, will have a bottleneck similar to the specific convex problem we analyzed here. Namely, for any such program to succeed we need a subgradient of the nuclear norm at $Y^\st$ which has a component $Z$ with spectral norm bounded by 1 (see the proof of Theorem \ref{thm:convex_recovery} in Appendix \ref{app:proof-convex}).
%For any subgradient constructed, this spectral norm will scale at least with $\sqrt{n}$, and this creates the same recovery bottleneck for all such optimization problems.  
%\red{hmm, still not a good sentence...}
%so the dominant term in \eqref{cond:convex_recovery} remains the same. 
%\red{I don't understand this heuristic, and which is the stated dominating term?-Roy}
For example, when all $p_k$ and $q$ are $O(1)$, this requires the minimum cluster size to be at least $O(\sqrt{n})\,$; also see Example \ref{ex:hard}. %\amin{this is a very special case, mentioned at the end of Example 6 as well. I think we should omit this remark. }

It is worth mentioning that for some community configurations, a simple counting argument can provide us with the exact underlying community structure; hence no need to solve a semidefinite program as above. We present one such algorithm in Appendix \ref{sec:simple} and characterize exact recovery guarantees. 

In the following, we attempt to provide a better picture of the model space in terms of recoverability. Section \ref{sec:hard} considers a modified maximum likelihood estimator to identify bigger parts of the model space that can be recovered exactly. Section \ref{sec:impossiblity} provides an information-theoretic argument to exclude part of the model space that are impossible to recover exactly. 

\subsection{Exactly Recoverable Models}\label{sec:hard}
Next, we consider an estimator, inspired by maximum likelihood estimation, and characterize a subset of the model space which is exactly recoverable via this simple estimation method. The proposed estimation approach is not computationally tractable and is only used to examine the conditions for which exact recovery is possible. 
For a fixed $Y \in \Yc$ and an observed matrix $A\,$, the likelihood function is given by
\[
\prob_Y(A)=\prod_{i<j}p^{A_{ij}Y_{ij}}_{\tau(i,j)}(1-p_{\tau(i,j)})^{(1-A_{ij})Y_{ij}}q^{A_{ij}(1-Y_{ij})}(1-q)^{(1-A_{ij})(1-Y_{ij})},
\]
where $\tau:\{1,\ldots,n\}^2\to \{1,\ldots,r\}$ and $\tau(i,j)=k$ if and only if $(i,j)\in \Cc_k\,$, and arbitrary in $\{1,\ldots,r\}$ otherwise. %\maryam{what do we mean by arbitrary?}. 
The log-likelihood function is given by
\[
\log \prob_Y(A)=\sum_{i<j}\log\frac{(1-q)p_{\tau(i,j)}}{q(1-p_{\tau(i,j)})}A_{ij}Y_{ij}+\sum_{i<j}\log\frac{1-p_{\tau(i,j)}}{1-q}Y_{ij} + \textrm{ terms not involving }\{Y_{ij}\} \,.
\]
Maximizing the log-likelihood involves maximizing a weighted sum of $\{Y_{ij}\}$'s where the weights depend on the (usually unknown) values of $q,p_1,\ldots,p_r\,$. To be able to work with less information, we will use the following modification of maximum likelihood estimation, which only uses the knowledge of $n_0,n_1,\ldots,n_r\,$. 
\beqa\label{proc:MLE-like}
\fbox{
\begin{minipage}[c][5em][c]{0.5\textwidth}{
{
%\refstepcounter{procedure}\label{proc:MLE-like} \bf Procedure \theprocedure
\bf Non-convex Recovery
} \vskip.6em
\begin{algorithmic}% [1]
\STATE {\bf input:} $\{n_k\}$ \\[.3em]
\STATE {\bf output:} $\hat Y = \arg\underset{Y}{\max}\; 
%\left\{\sum_{(i,j)\in\Omega} A_{ij}Y_{ij}:\; Y\in\Yc \right\}$
\left\{\sum_{i,j} A_{ij}Y_{ij}:\; Y\in\Yc \right\}$
\end{algorithmic}
}\end{minipage}}
\eeqa

\begin{theorem}\label{thm:hard-recovery}
Suppose $n_{\min}\geq 2$ and $n\geq 8\,$. % n \geq exp(2)
Under the heterogenous stochastic block model, provided that 
\beqas	%\label{eq:condition-recoverability}
\rho_{\min}  \geq 4 (17+\eta)\bigg( \tfrac{1}{3} + \frac{p_{\min}(1-  p_{\min}) +q(1-q)}{p_{\min}-q}  \bigg) \log n  \,,
\eeqas
for some choice of $\eta>0\,$, the optimal solution $\hat{Y}$ of the non-convex recovery program in \eqref{proc:MLE-like} coincides with $Y^\st$, with a probability not less than $1-5 \tfrac{p_{\max}-q}{p_{\min}-q}n^{2-\eta}\,$. 
\end{theorem}
%\maryam{specify for completeness: what values allowed for $\eta$? should be $>2$, anything else?}
Notice that $\rho_{\min} = \min_{k=1,\ldots,r}\, n_k(p_k-q)$ and $p_{\min} =  \min_{k=1,\ldots,r}\, p_k$ do not necessarily correspond to the same community. 

\subsection{When is Exact Recovery Impossible?}\label{sec:impossiblity}
\begin{theorem}	\label{thm:impossibility}
If any of the following conditions holds,
\begin{enumerate}[(1)]%[({A}1)]
\item \label{condn:impossible-first}
$2\leq n_k \leq n/e\,$, and
\begin{align}
4\sum_{i=1}^r n_k^2 \chidiv(p_k,q) \leq \tfrac{1}{2}\sum_k n_k\log \tfrac{n}{n_k}-r-2 
\end{align}
\item \label{condn:impossible-second}
$2\leq n_k \leq n/e\,$, and
\begin{align}
\tfrac{1}{2}r + \log \tfrac{1-  p_{\min}}{1-  p_{\max}}  +1 + \sum_k n_k^2p_k
\leq (\tfrac{1}{4}n - \sum n_k^2p_k) \log n + \sum (n_kp_k-\tfrac{1}{4})n_k \log n_k
\end{align}
\item \label{condn:impossible-third}
$n\geq 128\,$, $r\geq 2$ and %$n_{\min}\leq n/2\,$. 
\begin{align}
\max_k \; n_k\left(\chidiv(p_k,q)+\chidiv(q,p_k) \right) \leq \tfrac{1}{12}\log(n-n_{\min}) 
%\red{\chidiv(p_{\max},q)+\chidiv(q,p_{\max}) \leq \tfrac{1}{12}\log(n-n_{\min}) }
%\max_{i=1,\ldots,r}  \frac{n_k(p_k-q)^2}{q(1-q) \wedge p_k(1-p_k)}\leq \tfrac{1}{24}\log(n-n_{\min})
\end{align}
\end{enumerate}
then
$$\inf_{\hat{Y}}\sup_{Y^\st \in \Yc}\prob[\hat{Y}\neq Y^\st]\geq {1 \over 2}$$
where the infimum is taken over all measurable estimators $\hat{Y}$ based on the realization $A$ generated according to the heterogenous stochastic block model (HSBM).
\end{theorem}

\subsection{Partial Observations}	\label{sec:partial}
In the general stochastic block model, we assume that the entries of a symmetric adjacency matrix $A\in\{0,1\}^{n\times n}$ have been generated according to a combination of \ER models with parameters that depend on the true clustering matrix. 
In the case of partial observations, we assume that the entries of $A$ has been observed independently with probability $\gamma\,$. In fact, every entry of the input matrix falls into one of these categories: {\em observed as one} denoted by $\Omega_1$, {\em observed as zero} denoted by $\Omega_0$, and {\em unobserved} which corresponds to $\Omega^c$ where $\Omega=\Omega_0\cup\Omega_1\,$. 
If an estimator only takes the observed part of the matrix as the input, one can revise the underlying probabilistic model to incorporate both the stochastic block model and the observation model; i.e. a revised distribution for entries of $A$ as 
\beqas
A_{ij} = \begin{cases}
\ber(\gamma p_k) 	& (i,j)\in\Cc_k \text{ for some } k \\
\ber(\gamma q) 	& i\in\Cc_k \text{ and } j\in\Cc_l \text{ for }  k\neq l \,.
\end{cases}
\eeqas
yields the same output from an estimator that only takes in the observed values. 
%In the following, unless it is not easy to adapt for partial observations, we will avoid mentioning $\gamma\,$, $\Omega_0$ and $\Omega_1$ to simplify the notation. 
Therefore, the algorithms in \eqref{proc:convex-recovery} and \eqref{proc:MLE-like}, as well as the results of Theorems \ref{thm:convex_recovery}, \ref{thm:convex_recovery2}, \ref{thm:hard-recovery}, can be easily adapted to the case of partially observed graphs. 
%A simple observation can reduce the case of partial observation to the case of full information. Note that in the observed matrix, an entry is equal to one, if it was one in the drawn matrix and has been observed; otherwise it is equal to zero. This can be seen from the independence of values and observation. Hence, we can simply multiply all probabilities of 1 in the original problem by $\gamma$ and use previous results. 

\section{Future Directions}
%We have provided a series of extensions to prior work, especially \cite{chen2014statistical}. 
%\maryam{this section needs some editing.}
We have provided a series of extensions to prior work (especially \cite{chen2014statistical,abbe2015community}),
%\maryam{mention more than just one paper here}, 
however there are still interesting problems that remain open. Future directions for research on this topic include the following. 
 
\paragraph{Models for Partial Observation.} We considered the case where a subset of the edges in the underlying graph were observed uniformly at random. In practice, however, the observed edges are often not uniformly sampled, and care will be needed to model the effect of nonuniform sampling. 
Also, in many practical problems, the observed edges may be chosen by
the algorithm based on some prior information (non-adaptive), or based
on observations made so far (adaptive); e.g., see Yun and Proutiere \cite{yun2014community}. 
It will be interesting to examine what the algorithms can achieve in these scenarios.

% Yun and Proutiere \cite{yun2014community} present two kinds of algorithms; non-adaptive (when the edges that are revealed are pre-selected) and adaptive (when the algorithm can adaptively choose a (bounded) set of edges to explore). The first type of algorithms are similar to the ones studied here, though they are more general (we consider in this paper the case when either we have access to the entire adjacency matrix, or to a set of uniformly sampled entries). 

\paragraph{Overlapping Communities.}
SBMs with overlapping communities represent a more realistic model than the non-overlapping case; it has been shown that the large social and information network community structure is quite complex and that very large communities tend to have significant overlap. Only a few references in the literature have considered this problem (e.g., \cite{abbe2015community}), and  there are many open questions on recovery regimes and algorithms. 
%There is work in the literature considering several models that fall under this category with their own shortcomings; e.g. \cite{abbe2015community,,,,}\amin{we need to give a few refs}. 
It would be interesting to develop a convex optimization-based algorithm for recovery of models generated by SBM with overlapping communities.

\paragraph{Outlier Nodes.}
A practically important extension to the SBM is to allow for adversarial outlier nodes. 
Cai and Li in \cite{cai2014robust} proposed a semidefinite program that can recover the clusters in an SBM in the presence of outliner nodes connected to other nodes in an arbitrary way, provided that the number of outliers is small enough. Their result is comparable to the best known results in the case of balanced clusters and equal probabilities.
However, their complexity results are still parametrized by $p_{\min}$ and $n_{\min}$, which excludes useful examples, as discussed in Section \ref{sec:this-paper}. Extending our results to the setting of \cite{cai2014robust} is a direction for future work.

%\amin{we can squeeze some of this in the above: Cai and Li in \cite{cai2014robust} proposed a semidefinite program that can recover the clusters in an SBM in the presence of outliner nodes connected to other nodes in an arbitrary way, provided that the number of outliers is small enough.  A model that allows for outliers is a practically important extension to SBM, and theoretical analysis reveals that their approach is comparable to the best known results in the case of balanced clusters and equal probabilities. However, their complexity results are still parametrized by $p_{\min}$ and $n_{\min}$, which excludes useful examples, as discussed in Section \ref{sec:this-paper}. Extending our results to the setting of \cite{cai2014robust} is a direction for future work.}

%%%%%%%%%%%%%%%%%%%%%%%%%%%%%%%%%%%
\bibliographystyle{alpha}
\bibliography{CommunityDetection}
%%%%%%%%%%%%%%%%%%%%%%%%%%%%%%%%%%%
\appendix
\section{Proofs for Convex Recovery}	\label{app:proof-convex}
In the following, we present the proofs of Theorems \ref{thm:convex_recovery} and \ref{thm:convex_recovery2}.

\subsection{Proof of Theorem \ref{thm:convex_recovery}} \label{app:proof-convex_recovery}
We are going to prove that under the HSBM, with high probability, the output of the convex recovery program in \eqref{proc:convex-recovery} coincides with the underlying clustering matrix $Y^\st$ provided that
\beqa%\label{cond:convex_recovery}
\rho_k^2 &\gtrsim     n_kp_k(1-p_k)\log n_k \\
(p_{\min}-q)^2 &\gtrsim     q(1-q)\tfrac{\log n_{\min}}{n_{\min}} \\
\rho_{\min}^2 &\gtrsim     \max \left\{ \max_k \; n_kp_k(1-p_k),\, nq(1-q),\, \log n \right\} 
\eeqa
as well as $\sum_{k=1}^r n_k^{-\alpha}  =  o(1)$ for some $\alpha>0\,$. 
Notice that $p_k(1-p_k)n_k\gtrsim \log n_k\,$, for all $k=1,\ldots,r\,$, is implied by the first condition, as mentioned in Remark \ref{rem:connected-convex}.
 
Before proving Theorem \ref{thm:convex_recovery}, we state a crucial
result from random matrix theory that allows us to bound the spectral
radius of the matrix $A - \mathbb{E}(A)$ where $A$ is an instance of adjacency matrix under HSBM. This result appears, for
example, as Theorem 3.4 in \cite{chatterjee2012matrix}\footnote{As a
  more general result about the norms of rectangular matrices, but
  with the slightly stronger growth condition $\sigma^2 \geq
  \log^{6+\epsilon}n /n$.}. Although Lemma 2 from
\cite{tomozei2010distributed} appears to state a weaker version of
this result, the proof presented there actually supports the version
we give below in Lemma \ref{randmat}. Finally, Lemma
8 from \cite{vu2014simple} states the same result and presents a very brief
sketch of the proof idea, along the lines of the proof presented fully
in \cite{tomozei2010distributed}.

\begin{lemma} \label{randmat} 
  Let $A = \{a_{ij}\}$ be a $n \times n$ symmetric random
  matrix such that each $a_{ij}$ represents an independent random
  Bernoulli variable with $\E(a_{ij}) = p_{ij}\,$. Assume that there
  exists a constant $C_0$ such that $\sigma^2 =  \max_{i,j}
  p_{ij}(1-p_{ij}) \geq C_0 \log n /n$. Then for each constant $C_1>0$ there
  exists $C_2>0$ such that 
\[
\prob \left ( \norm{A - \E(A)}  \geq C_2 \sigma \sqrt{n} \right)
~\leq~ n^{-C_1} \,.
\]
\end{lemma}

As an immediate consequence of this, we have the following corollary.

\begin{corollary} \label{weakrandmat} 
  Let $A = \{a_{ij}\}$ be a $n \times n$ symmetric random
  matrix such that each $a_{ij}$ represents an independent random
  Bernoulli variable with $\E(a_{ij}) = p_{ij}\,$. Assume that there
  exists a constant $C_0$ such that $\sigma^2 = \max_{i,j}
  p_{ij}(1-p_{ij}) \leq C_0 \log n /n\,$.  Then for each constant $C_1>0$ there
  exists $C_3>0$ such that such that 
\[
\prob \left ( \norm{A - \E(A)}  \geq C_3 \sqrt{\log n} \right)
~\leq~ n^{-C_1}\,.
\]
\end{corollary}

\begin{proof} The corollary follows from Lemma \ref{randmat}, 
by replacing the $(1,1)$ entry of $A$ with a Bernoulli variable
  of probability $p_{11} = C_0 \log n /n$. Given that the old $(1,1)$ entry and the new $(1,1)$ entry are both Bernoulli variables, this can change $\norm{A -
  \E(A)}$ by at most $1$. 
The new maximal variance is equal to 
  %$\sigma_{\operatorname{new}}^2 = 
  $\max_{i,j} p_{ij}(1-p_{ij}) =  C_0 \log n/n\,$. Therefore Lemma \ref{randmat}
  is applicable to the new matrix and the conclusion holds. 
\end{proof}

% some results from random matrix theory literature which will be used throughout the proof.
% \begin{lemma}[Theorem 3.4 \cite{chatterjee2012matrix}]\label{lem:chatterjee}
% Fix $0\leq m\leq n$ and let $A=(a_{ij}) \in \mathbb{R}^{m\times n}$ be a random matrix whose entries are independent random variables with $\E(a_{ij})=0$ and $\sup_{i,j}\E(a_{ij}^2)\leq \sigma^2$. Furthermore, assume $a_{ij}\leq 1$ almost surely. For any $\epsilon>0$, suppose that $\sigma^2\geq \log^{6+\epsilon} n/n\,$. Then for any $\eta(0,1)\,$, we can find $C_1=C_1(\epsilon,\eta)$ and $C_2=C_2(\eta)$ such that
% \[
% \prob(\norm{A}>(2+\eta)\sigma\sqrt{n})\leq C_1\exp(-C_2\sigma^2n) \,.
% \]
% \end{lemma}
% \begin{lemma}[Lemma 2 \cite{tomozei2010distributed}]\label{lem:massoulie}
% Let $A$ be a $n\times n$ symmetric random matrix with zero diagonal entries, and $A_{ij}$'s, for $i\leq j\,$, are independent Bernoulli variables with parameter $p_{ij}=a_{ij}\omega/n\,$. Here $a_{ij}$'s are constants and $\omega=\Omega(\log n)\,$. Then for each $C_3>0\,$, we can find $C_4$ such that with probability at least $1-n^{-C_3}\,$, we have the following bound for spectral norm of the centered matrix $\norm{A-\E[A]}\leq C_4 \sqrt{\omega}\,$.
% \end{lemma}

% \textcolor{red}{add a note about Massoulie's result:  their result is actually more general than appears}

We use Lemma \ref{randmat} to prove the following result.
\begin{lemma}\label{lem:specnorm-convex-recovery}
Let $A$ be generated according to the heterogenous stochastic block
model (HSBM). Suppose
\begin{enumerate}[(1)]
%\item $\max_i p_i\leq 1-o(1)\,$,
\item\label{condn1:specnorm-convex-recovery} 
    $p_k(1-p_k)n_k \gtrsim \log n_k\,$, for $k=1,\ldots, r\,$, and
\item\label{condn2:specnorm-convex-recovery} 
    there exists an $\alpha>0$ such that $\sum_{k=0}^r n_k^{-\alpha}
  = o(1)\,$.
%$\max\{r\exp(-O(\min p_i(1-p_i)n_i)),\exp(-O(q(1-q)n)), \sum_i n_i^{-O(1)}\}\leq O(n^{-O(1)})\,$.
\end{enumerate} 
Then with probability at least $1-o(1)$ we have 
	\[
	%\norm{A-\E(A)} \lesssim \sqrt{\max_i p_i(1-p_i)n_i+ \max \{q(1-q)n, \log n\} } \,. 
	\norm{A-\E(A)} \;\lesssim\; \max_i\sqrt{p_i(1-p_i)n_i}+ \sqrt{\max \{q(1-q)n\,,\, \log n\} } \,. 
	\]
\end{lemma}

\begin{proof}
We split the matrix $A$ into two matrices, $B_1$ and $B_2\,$. $B_1$
consists of the block-diagonal projection onto the clusters, and $B_2$
is the rest. Denote the blocks on the diagonal of $B_1$ by $C_1$,
$C_2$, $\ldots, C_r$, where $C_i$ corresponds to the $i$th
cluster. Then $\norm{B_1 - \E(B_1)} = \max_i \norm{C_i - \E(C_i)}$,
and for each $i$, $\norm{C_i - \E(C_i)} \gtrsim \sqrt{p_i(1-p_i)
  n_i}$ with probability at most $n_i^{-\alpha}\,$, by Lemma
\ref{randmat}. 
By assumptions \eqref{condn1:specnorm-convex-recovery} and \eqref{condn2:specnorm-convex-recovery} of Lemma \ref{lem:specnorm-convex-recovery} and applying a union bound, we conclude that
\[
\norm{B_1 - \E(B_1)} \lesssim \max_i \sqrt{p_i(1-p_i)n_i} 
\]
with probability at least $1 - \sum_{i=1}^{r} n_i^{-\alpha} =
1-o(1)\,$. 
We shall now turn our attention to $B_2\,$. Let $\sigma^2 =
\max \{q(1-q), \log n/n \}\,$. By Corollary \ref{weakrandmat},
  $\norm{B_2 - \E(B_2)} \lesssim \max \{\sqrt{q(1-q)n}, \sqrt{\log n}
  \}\,$, with high probability. 
Putting the two norm estimates together, the conclusion of Lemma
\ref{lem:specnorm-convex-recovery} follows.
\end{proof}

We are now in the position to prove Theorem \ref{thm:convex_recovery}.

\begin{proof}[of Theorem \ref{thm:convex_recovery}]
We need to show that for any feasible $Y\neq Y^\st\,$, we have $\Delta(Y):=\langle A , Y^\st-Y \rangle>0\,$. Rewrite $\Delta(Y)$ as
\beqas	\label{eq:DeltaYsplit}
\Delta(Y)=\langle{A},{Y^\st-Y}\rangle=\langle{\E[A]},{Y^\st-Y}\rangle+\langle{A-\E[A]},{Y^\st-Y}\rangle \,.
%=\langle{\E[A]},{Y^\st-Y}\rangle+\lambda\langle{W},{Y^\st-Y}\rangle\,.
\eeqas
%where $W=(A-\E[A])/\lambda$, and $\lambda$ is a constant to be
%specified. 
%Since $\sum_{i,j}Y_{ij}=\sum_{k}n_k^2=\sum_{i,j}Y^\st_{ij}$ as required by the convex program, we have
%
Note that $\sum_{i,j} Y^\st_{ij} = \sum_{i,j} Y_{ij} = \sum_{k=1}^r
n_k^2$, thus $\sum_{i,j} (Y^\st_{ij} - Y_{ij}) = 0\,$. Express this as
\[
\sum_{k=1}^r \sum_{i,j \in V_k} (Y^\st-Y)_{ij} = - \sum_{k' \neq
  k''} \sum_{i \in V_{k'},\, j \in V_{k''}} (Y^\st-Y)_{ij} \,.
\]
Then we have
\beqas 
\langle{\E[A]},{Y^\st-Y}\rangle & = & \sum_{k=1}^r \sum_{i,j \in V_k^\st} p_k
(Y^\st-Y)_{ij} + \sum_{k' \neq k''} \sum_{i \in V_{k'},\, j \in V_{k''}} q (Y^\st-Y)_{ij} 
& = & \sum_{k=1}^r \sum_{i,j \in V_k} (p_k - q) (Y^\st-Y)_{ij} \,.
\eeqas
Finally, since $0 \leq (Y^\st-Y)_{ij} \leq 1$ for $i,j \in V_k\,$, we can write
\beqa	\label{doi}
\langle{\E[A]},{Y^\st-Y}\rangle= \sum_{k=1}^r \sum_{i,j \in V_k} (p_k -q)
\norm{(Y^\st-Y)_{\Cc_k}}_1 \,.
\eeqa
%By condition (\ref{cond:convex_recovery}), we know that $n_k=O(n)$ and
%
%
%We now construct a subgradient of $\norm{Y^\st}_\st$ as follows: let
Next, recall that the subdifferential (i.e., the set of all subgradients) of $\norm{\cdot}_\st$ at $Y^\st$ is given by 
\[
\partial \norm{Y^\st}_\st = \{ UU^T + Z ~ \big |~ U^TZ = ZU = 0\,,\; \norm{Z}\leq 1\} 
\]
where $Y^\st=UKU^T$ is the singular value decomposition for $Y^\st$ with 
$U\in\mathbb{R}^{n\times r}\,$, $K=\operatorname{diag}(n_1,\ldots,n_r)\,$, and  
$U_{ik} = 1/\sqrt{n_k}$ if node $i$ is in cluster $\Cc_k$ and
$U_{ik}=0$ otherwise.

Let $M := A - \E[A]\,$. Since conditions \eqref{condn1:specnorm-convex-recovery} and \eqref{condn2:specnorm-convex-recovery} of Lemma \ref{lem:specnorm-convex-recovery} are
verified, there exists $C_1>0$ such that $\norm{M}\leq \lambda\,$, with probability $1-o(1)\,$, where   
\begin{eqnarray} \label{lambdadef}
\lambda  :=  C_1 \left(\max_i \sqrt{p_i(1-p_i)n_i}+ \sqrt{\max \{q(1-q)n, \log n\} } \right)\,. 
\end{eqnarray} 
Furthermore, let the projection operator onto a subspace $T$ be defined by
\[
\mathcal{P}_T(M):=UU^TM+MUU^T-UU^TMUU^T \,,
\]
and also $\mathcal{P}_{T^\bot}=\mathcal{I}-\mathcal{P}_T$, where $\mathcal{I}$ is the
identity map. 
Since $\norm{\mathcal{P}_{T^\bot}(M)}\leq \norm{M}\leq \lambda$ with high probability,
$UU^T+ \tfrac{1}{\lambda}\mathcal{P}_{T^\bot}(M)\in \partial \norm{Y^\st}_{\st}$ with high probability. 
Now, by the constraints of the convex program, we have
\beqa
0\geq \norm{Y}_{\st}-\norm{Y^\st}_{\st}
\geq \langle{UU^T+\tfrac{1}{\lambda}\mathcal{P}_{T^\bot}(M)},{Y-Y^\st}\rangle
= \langle{UU^T-\tfrac{1}{\lambda}\mathcal{P}_{T}(M)},{Y-Y^\st}\rangle+\tfrac{1}{\lambda}\langle{M},{Y-Y^\st}\rangle \,,
\eeqa
which implies $\langle{M},{Y^\st-Y}\rangle\geq \langle{\mathcal{P}_{T}(M)-\lambda UU^T},{Y^\st-Y}\rangle\,$. 
Combining \eqref{eq:DeltaYsplit} and \eqref{doi} we get,
\beqa \label{trei}
\Delta(Y)  &\geq \sum_{k=1}^r (p_k-q)\norm{(Y^\st-Y)_{\Cc_k}}_1+ \langle{\mathcal{P}_{T}(M)-\lambda UU^T},{Y^\st-Y}\rangle\\ 
&\geq \sum_{k=1}^r (p_k-q)\norm{(Y^\st-Y)_{\Cc_k}}_1\\
&\quad - \sum_{k=1}^r \underbrace{\norm{(\mathcal{P}_{T}(M)-\lambda UU^T)_{\Cc_k}}_{\infty}}_{(\mu_{kk})}\norm{(Y^\st-Y)_{\Cc_k}}_1\\
&\quad - \sum_{k'\neq k''}\underbrace{\norm{(\mathcal{P}_{T}(M)-\lambda UU^T)_{V_{k'}\times V_{k''}}
  }_{\infty}}_{(\mu_{k' k''})}\norm{(Y^\st-Y)_{V_{k'}\times V_{k''}}}_1
%&\geq \sum_{k=1}^r \bigg((p_k-q)-2\lambda \norm{UU^T}_{\infty}-2\norm{\mathcal{P}_T(\lambda W)}_{\infty}\bigg)\norm{(Y^\st-Y){\bf 1}_k}_1\\
%&\geq \sum_{k=1}^r \bigg((p_k-q)-{2\lambda \over n_{\min}}-2\norm{\mathcal{P}_T(\lambda W)}_{\infty}\bigg)\norm{(Y^\st-Y){\bf 1}_k}_1,\\
\eeqa
where we have made use of the fact that an inner product can be
bounded by a product of dual norms. 
We now derive bounds for the quantities $\mu_{kk}$ and
$\mu_{k'k''}$ marked above. Note that the former indicates sums over
the clusters, while the latter indicates sums outside the
clusters.  

For $\mu_{kk}$, if $(i,j)\in \Cc_k$ then
\beqas
\left(\mathcal{P}_T(M)-\lambda UU^T\right)_{ij}
&= \left(UU^TM+MUU^T-UU^TMUU^T-\lambda UU^T\right)_{ij}\\
&={1 \over n_k} \sum_{l \in \Cc_k}M_{lj}+{1 \over n_k}\sum_{l \in \Cc_k}M_{il}-{1 \over n_k^2}\sum_{l,l'\in \Cc_k}M_{ll'}-{\lambda \over n_k} \,.
\eeqas
Recall Bernstein's inequality (e.g. see Theorem 1.6.1 in \cite{tropp2015introduction}):
\begin{proposition} (Bernstein Inequality) Let $S_1, S_2$, $\ldots, S_n$ be independent, centered, real random variables, and assume that each one is uniformly bounded:
\[
\E[S_{k}] = 0~~\text{and}~~\abs{S_k}\leq L~~\text{for each } k=1, \ldots,
  n \,.
\]
Introduce the sum $Z= \sum_{k=1}^n S_k$, and let $\nu(Z)$ denote the
variance of the sum:
\[
\nu(Z) = \E[Z^2] = \sum_{k=1}^n \E[S_k^2] \,.
\]
Then 
\[
\prob[~\abs{Z} \geq t~]~\leq~2 \exp\left(\frac{-t^2/2}{\nu(Z) +
    Lt/3}\right)~~\text{for all } t \geq 0 \,.
\]
\end{proposition}

We will apply it to bound the three sums in $\mu_{kk}$, using the fact that each of
the sums contains only centered, independent, and bounded
variables, and that the variance of each entry in the sum is $p_k(1-p_k)\,$. 
For the first two sums, we can use $t \sim \sqrt{n_k p_k(1-p_k)
  \log n_k}$ to obtain a combined failure probability (over the entire
  cluster) of $O(n_k^{-\alpha})$. Finally, for the third sum, we may
  choose $t \sim n_k \sqrt{p_k (1-p_k) \log n_k}$, again for a combined failure
  probability over the whole cluster of no more than $O(n_k^{-\alpha})$.

We have thusly
\beqas
\mu_{kk}\leq \abs{ \tfrac{1}{n_k} \sum_{l \in \Cc_k}M_{lj} }
+ \abs{ \tfrac{1}{n_k}\sum_{l \in \Cc_k}M_{il} }
+ \abs{ \tfrac{1}{n_k^2}\sum_{l,l'}M_{l,l'} } +{\lambda \over n_k}
\lesssim \sqrt{\frac{p_k(1-p_k)}{n_k} \log n_k}+{\sqrt{p_k (1-p_k) \log n_k} \over n_k}+{\lambda \over n_k}~,\\
\eeqas
for all $i,j \in \Cc_k$, with probability $1-O(n_k^{-\alpha})$. Note
that in the inequality above, the second term is much smaller in
magnitude than the first, so we can disregard it; using
\eqref{lambdadef}, we obtain
\beqa	\label{eq:Ak-bound}
\mu_{kk} \lesssim \frac{1}{n_k} \left ( \sqrt{n_k p_k(1-p_k)
    \log n_k} + \max_i \sqrt{p_i (1-p_i) n_i} + \sqrt{\max \{q(1-q)n, \log n\} }
\right )\,, 
\eeqa
and by taking a union bound over $k$ we can conclude that the
probability that any of these bounds fail is
$o(1)\,$. 
Similarly, for $\mu_{k'k''}$, for $k'\neq k''\,$, we can calculate that
\begin{eqnarray} \label{eq:Bkk-bound-before}
\mu_{k'k''}\leq \abs{\tfrac{1}{n_{k'}} \sum_{l \in
    \Cc_{k'}}M_{lj}}+\abs{\tfrac{1}{n_{k''}}\sum_{l \in
    \Cc_{k''}}M_{il}}+\abs{\tfrac{1}{n_{k'}n_{k''}}\sum_{l'\in\Cc_{k'},\,l''\in\Cc_{k''}}M_{l',l''}}  & \lesssim & 
\\
& & \hspace{-3.6cm} 
\sqrt{q(1-q) (\frac{\log n_{k'}}{n_{k'}} +\frac{\log
  n_{k''}}{n_{k''}})}+ \frac{\sqrt{q(1-q) \log (n_{k'}
  n_{k''})}}{\sqrt{n_{k'} n_{k''}}}~, \nonumber
%(n_{k'}^{-1}+n_{k''}^{-1})\log n\\
\end{eqnarray}
with failure probability over all $i \in \Cc_{k'}$, $j \in \Cc_{k''}$
of no more than $O(n_{k'}^{-\alpha} n_{k''}^{-\alpha})\,$. We do this by
taking $t \sim \sqrt{n_{k'} q(1-q) \log (n_{k'} n_{k''})}$,
respectively $t \sim \sqrt{n_{k''} q(1-q) \log (n_{k'} n_{k''})}$ in
the first two sums. For the third, we just take $t \sim \sqrt{ n_{k'} n_{k''} q(1-q) \log (n_{k'} n_{k''})}\,$. 
As before, note that the second term is much smaller in magnitude than
the first, and hence we can disregard it to obtain 
\begin{eqnarray}	\label{eq:Bkk-bound}
\mu_{k'k''}\lesssim \max_{k} \sqrt{\frac{q(1-q)\log
    n_k}{n_k}} = \sqrt{\frac{q(1-q) \log n_{\rm{min}}}{n_{\rm{min}}}} := \mu_{\operatorname{off}}~,
\end{eqnarray}
as the function $\log x/x$ is strictly increasing if $x \geq 3$, with
the probability that all of the above are simultaneously true being
$1-o(1)$.  
Since the bound on $\mu_{k'k''}$ is independent of $k'$ and $k''$ we can rewrite \eqref{trei} as
\beqas
\Delta(Y)
&\geq \sum_{k=1}^r (p_k-q)\norm{(Y^\st-Y)_{\Cc_k}}_1
 - \sum_{k=1}^r \mu_{kk} \norm{(Y^\st-Y)_{\Cc_k}}_1
 - \sum_{k'\neq k''} \mu_{k'k''}\norm{(Y^\st-Y)_{V_{k'}\times V_{k''}}}_1 \\
&\geq \sum_{k=1}^r \left(p_k-q - \mu_{kk} - \mu_{\operatorname{off}} \right) \norm{(Y^\st-Y)_{\Cc_k}}_1
\eeqas
where we use the fact that $\sum_{k' \neq k''} \norm{(Y^\st-Y)_{V_{k'}\times V_{k''}}}_{1} = \sum_{k=1}^r \norm{(Y^\st-Y)_{\Cc_k}}_{1}\,$. 
Finally, the conditions of theorem guarantee the nonnegativity of the right hand side, hence the optimality of $Y^\st$ as the solution to the convex recovery program in \eqref{proc:convex-recovery}. 
%We can now insert these upper bounds in \eqref{trei}, to obtain
%\beqas \small
%\Delta (Y) \geq \sum_{k=1}^r \Bigg( p_k - q 
%    &- \frac{1}{n_k} \left(\sqrt{n_k p_k(1-p_k) \log n_k} + \max_i \sqrt{p_i (1-p_i) n_i} + \sqrt{q(1-q)n \vee \log n} \right)  \\
%    &- \sqrt{\frac{q(1-q) \log n_{\rm{min}}}{n_{\rm{min}}}} \Bigg)
%\norm{(Y^\st-Y)_{\Cc_k}}_{1} \,,
%\eeqas
%the above follows from the fact that we have upper bounded all quantities $(\mathcal{B}_{k'k''})$ and have used the fact that $\sum_{k' \neq k''} \norm{(Y^\st-Y)_{V_{k'}\times V_{k''}}}_{1} = \sum_{k=1}^r \norm{(Y^\st-Y)_{\Cc_k}}_{1}$. Finally, as the conditions imposed in the theorem insure that the right hand side is positive, this gives the final result. 
\end{proof}

\subsection{Proof of Theorem \ref{thm:convex_recovery2}}
We use a different result than Lemma \ref{lem:specnorm-convex-recovery}, which we state below. 
\begin{lemma}[Corollary 3.12 in \cite{bandeira2014sharp}]	\label{lem:bandeira2014sharp}
Let $X$ be an $n\times n$ symmetric matrix whose entries $X_{ij}$ are independent symmetric random variables. Then there exists for any $0<\epsilon \leq \tfrac{1}{2}$ a universal constant $c_\epsilon$ such that for every $t\geq 0$
\[
\norm{X}\leq 2(1+\epsilon) \tilde \sigma + t \,,
\]
with probability at least $1-n\exp(\tfrac{-t^2}{c_\epsilon \tilde \sigma_\st^2})\,$, where 
\[
\tilde \sigma  = \max_i \sqrt{\sum_j \E[X_{ij}^2]} \;,\quad \tilde \sigma_\st = \max_{i,j} \norm{X_{ij}}_\infty \,.
\]
\end{lemma}
We specialize Lemma \ref{lem:bandeira2014sharp} to HSBM to get the following result. 
\begin{lemma}\label{lem:specnorm-convex-recovery2}
Let $A$ be generated according to the heterogenous stochastic block model (HSBM). Then there exists for any $0<\epsilon \leq \tfrac{1}{2}$ a universal constant $c_\epsilon$ such that
\[
\norm{A-\E(A)} \leq 4(1+\epsilon) \max\{\sigma_{\max}, \sigma_0\} + \sqrt{2 c_\epsilon \log n} 
\]
with probability at least $1-n^{-1}\,$. 
\end{lemma}

We can now present the proof for Theorem \ref{thm:convex_recovery2}.

\begin{proof}The proof follows the same lines as the proof of Theorem
\ref{thm:convex_recovery}. Given the similarities between the proofs, we will only describe here
the differences between the tools employed, and how they affect the
conditions in Theorem \ref{thm:convex_recovery2}. 
The proof proceeds identically as before, up to the definition of
$\lambda$, which--since we use Lemma
\ref{lem:specnorm-convex-recovery2} rather than
\ref{lem:specnorm-convex-recovery}--becomes
\begin{eqnarray} \label{lambdadef1}
\lambda  :=  C_2 \max \{\sigma_{\max},\, \sigma_0,\, \sqrt{\log n}\}  \,,
\end{eqnarray}
where $C_2$ was chosen as a good upper bounding constant for Lemma
\ref{lem:specnorm-convex-recovery2}.

The other two small changes come from the fact that we will need to
make sure that the failure probabilities for the quantities
$\mu_{kk}$ and $\mu_{k'k''}$ are polynomial in $1/n$, which
leads to the replacement of $\log n_k$ in either of them by a $\log
n$. The rest of the proof proceeds exactly in the same way. 
%After absorbtion of the $\sqrt{p_i (1-p_i)n_i \log n}$ term by the $\max_i \sqrt{p_i(1-p_i)n_i \log n} + \sqrt{q(1-q)n }\,$, one obtains the statement of the theorem. 
\end{proof}

\section{Proofs for Recoverability and Non-recoverability}	\label{app:proof-rec}

\subsection{Proofs for Recoverability}

\begin{proof}[of Theorem \ref{thm:hard-recovery}]
For $\Delta(Y):=\langle{A},{Y^\st-Y}\rangle\,$, we have to show that for any feasible $Y\neq Y^\st\,$, we have $\Delta(Y)>0\,$. For simplicity we assume $Y_{ii}=Y_{ii}^\st=0$ for all $i \in\{1,\ldots,n\}\,$. Consider an splitting as
\begin{align}	\label{eq:split-iprod}
\Delta(Y)=\langle{A},{Y^\st-Y}\rangle=\langle{\E(A)},{Y^\st-Y}\rangle+\langle{A-\E(A)},{Y^\st-Y}\rangle\,.
\end{align}
Notice that $Y^\st=  \sum_{k=1}^r \mathbf{1}_{\Cc_k}$ and $\E(A) = q\mathbf{1}\mathbf{1}^T + \sum_{k=1}^r (p_k-q)\mathbf{1}_{\Cc_k}\,$. Considering $d_k(Y)=\langle{Y^\st_{\Cc_k}},{Y^\st-Y}\rangle\,$, the number of entries on $\Cc_k$ on which $Y$ and $Y^\st$ do not match, we get 
\begin{align}	\label{eq:split-EA-part}
\langle{\E(A)},{Y^\st-Y}\rangle =  \sum_{k=1}^r (p_k-q) d_k(Y)
\end{align}
where we used the fact that $Y,Y^\st\in\mathcal{Y}$ and have the same number of ones and zeros, hence $\sum_{i,j} Y_{ij} = \sum_{i,j} Y^\st_{ij}\,$.  
On the other hand, the second term in \eqref{eq:split-iprod} can be represented as
\begin{align*}
T(Y):=\langle{A-\E(A)},{Y^\st-Y}\rangle
\;=\; \sum_{Y_{ij}^\st=1,Y_{ij}=0} (A-\E(A))_{ij}
\;+\; \sum_{Y_{ij}^\st=0,Y_{ij}=1} (\E(A)-A)_{ij} 
\end{align*}
where each term is a centered Bernoulli random variable bounded by $1\,$. Observe that the total variance for all the summands in the above is given by 
\begin{align*}
\sigma^2=\sum_{k=1}^r d_k(Y)  p_k(1-  p_k) + q(1-  q) \sum_{k=1}^r d_k(Y)  \,.
\end{align*}
Then, combining \eqref{eq:split-iprod} and \eqref{eq:split-EA-part}, and applying the Bernstein inequality yields
\begin{align*}
\prob(\Delta(Y)\leq 0) = 
\prob\bigg(T(Y)\leq -\sum_k  (p_k-q)d_k(Y)\bigg)
\leq \exp\bigg( -\frac{ t^2}{2\sigma^2+ 2t/3}  \bigg) 
&\leq \exp\bigg(-\frac{\sum_k   (p_k-q) d_k(Y)}{2\nu(Y) + 2/3}\bigg)
\end{align*}
where $t = \sum_k  (p_k-q)d_k(Y)$ and
\begin{align*}
\nu(Y) = \frac{\sigma^2}{t} 
= \frac{\sum_{k=1}^r  ( p_k(1-  p_k) +q(1-q))  d_k(Y)  }{\sum_k  (p_k-q)d_k(Y)}  
\leq \max_k\frac{p_k(1-  p_k) +q(1-q)}{p_k-q}  = \frac{p_{\min}(1-  p_{\min}) +q(1-q)}{p_{\min}-q}  :=\bar\nu_0  \,.
\end{align*}
Considering $\bar \nu:= 2\bar\nu_0 +2/3$ and $\theta_k:=\lfloor\frac{p_{k}-q}{p_{\min}-q} \rfloor\,$, we get 
\begin{align}	\label{eq:bound-DeltaY}
\prob(\Delta(Y)\leq 0)  
\leq \exp\bigg(-\tfrac{1}{\bar\nu} {\sum_k   (p_k-q) d_k(Y)}\bigg)
\leq \exp\bigg(-\tfrac{1}{\bar\nu} {(p_{\min}-q) \sum_k   \theta_k d_k(Y)}\bigg)
\end{align}
which can be bounded using the next lemma which is a direct extension of Lemma 4 in \cite{chen2014statistical}.

\begin{lemma}	\label{lem:thetad-bound}
Given the values of $\theta_k$ and $n_k\,$, for $k=1,\ldots,r\,$, and for each integer value $\xi\in [\min\theta_k(2n_k-1),\,\sum_{k}\theta_kn_k^2]\,$, we have
\begin{align}\label{eq:lem:thetad-bound}
\big| {\{[Y]\in \mathcal{Y}:\sum_{k=1}^r \theta_k d_k(Y)=\xi\}} \big|
\leq  \left(\frac{4\xi}{\tau}\right)^2  {n}^{16\xi/\tau}
\end{align}
where $\tau:= \min_k\,\theta_k n_k\,$, 
and $[Y] = \{Y'\in\mathcal{Y}:\; Y'_{ij} Y^\st_{ij} = Y_{ij} Y^\st_{ij} \}\,$.  
\end{lemma}
%Observe that the right hand side of \eqref{eq:lem:thetad-bound} is decreasing in $\tau$ for $\tau>0\,$. Therefore, we can replace $\tau$ with 
%\[
%\tau':= \frac{\rho_{\min}}{2(p_{\min}-q)} \leq \tau 
%\]
%to upper bound the probability in \eqref{eq:lem:thetad-bound}. 

Now plugging in the result of Lemma \ref{lem:thetad-bound} into \eqref{eq:bound-DeltaY} yields,
\begin{align}
\prob\bigg(\exists Y\in \mathcal{Y}:Y\neq Y^\st,\Delta(Y)\leq 0\bigg)
&\leq \sum_{\xi}\prob\big(\exists Y\in \mathcal{Y}: \sum_k \theta_k d_k(Y)=\xi\,,\,\Delta(Y)\leq 0\big) \nonumber\\
&\leq 2 \sum_\xi \left(\frac{4\xi}{\tau}\right)^2  { n}^{16\xi/\tau}  \exp\bigg(-\tfrac{1}{\bar\nu}  {(p_{\min}-q) \xi}\bigg)  \nonumber\\
&= 32  \sum_\xi \bigg(\frac{\xi}{\tau}\bigg)^2  \exp\bigg( (16\log n - \tfrac{1}{\bar\nu} (p_{\min}-q)\tau ) \frac{\xi}{\tau} \bigg)    \nonumber\\
&\leq 32  \sum_\xi \bigg(\frac{\xi}{\tau}\bigg)^2  \exp\bigg( (16\log n - \tfrac{1}{2\bar\nu} \rho_{\min} ) \frac{\xi}{\tau} \bigg) \label{eq:xi-decr-func} 
\end{align}
%where we used $\rho_{\min} - p_{\min}n_{\min}\leq (p_{\min}-q)\min\theta_kn_k \leq \rho_{\min}\,$. 

In order to have a meaningful bound for the above probability, we need the exponential term in \eqref{eq:xi-decr-func} to be decreasing. Hence, we require
$\rho_{\min}  \geq 64 \bar\nu \log n \,$.
Moreover, the function in \eqref{eq:xi-decr-func} is a decreasing function of $\xi/\tau$ for
\begin{align}\label{eq:xi-dec}
\frac{\xi}{\tau} \geq \frac{4\bar\nu }{ \rho_{\min} -32 \bar\nu \log n} \,.
\end{align}
Since $\xi\geq \min \theta_k(2n_k-1)\geq \min\theta_kn_k=\tau\,$, requiring the following condition (for some $\eta>0$ which will be determined later),
\begin{align}	\label{eq:cond-rho-min-1}
\rho_{\min}  \geq 2(16+\eta)\bar\nu\log n +4\bar\nu \,, %\frac{2\min\theta_kn_k}{\min\theta_k(2n_k-1)} \bar\nu 
\end{align}
implies 
\[
\frac{\xi}{\tau} \geq 1 \geq \frac{4}{4+2\eta \log n} 
\geq \frac{4\bar\nu }{ \rho_{\min} -32 \bar\nu \log n} 
\]
and allows us to bound the summation in \eqref{eq:xi-decr-func} with the largest term (corresponding to the smallest value of $\xi/\tau\,$, or an even smaller value, namely 1) times the number of summands; i.e.,
\begin{align}
\text{\eqref{eq:xi-decr-func}}
&\leq 32\;(\sum\theta_kn_k^2)\, \exp\left(16\log n - \tfrac{1}{2\bar\nu} \rho_{\min}\right) \\
%&\leq 50(\sum\theta_kn_k^2)\exp\bigg(20\log n- \tfrac{1}{\bar\nu}  (\rho_{\min}-p_{\min}n_{\min}) \bigg)\\
&\leq 32\sum\theta_kn_k^2 \exp(-2-\eta \log n) \\ %\exp\bigg(- \frac{4\min\theta_kn_k}{\min\theta_k(2n_k-1)} \bigg)\\
%\leq 7 \sum\theta_kn_k^2 n^{-\eta}
&\leq 5\, \theta_{\max} n^{2-\eta} \\
&\leq 5\, \tfrac{p_{\max}-q}{p_{\min}-q} n^{2-\eta} \,,
\end{align}
or, similarly,
\begin{align}
\text{\eqref{eq:xi-decr-func}}
\leq 32\sum\theta_kn_k^2 \exp(-2-\eta \log n) \leq 5 \frac{\sum_{k=1}^r\rho_k}{p_{\min}-q}n^{1-\eta} \,.
\end{align}
Hence, if the condition in \eqref{eq:cond-rho-min-1} holds we get the optimality of $Y^\st$ with a probability at least equal to the above. Finally, $n\geq 8$ implies $\log n\geq 2$ and \eqref{eq:cond-rho-min-1} can be insured by
\[
\rho_{\min} \geq 4(17+\eta) \bigg( \frac{1}{3}+ \frac{p_{\min}(1-p_{\min})+q(1-q)}{p_{\min}-q}\bigg) \log n \,.
\]
\end{proof}
\begin{proof}[of Lemma \ref{lem:thetad-bound}] We extend the proof of Lemma 4 in \cite{chen2014statistical} to our case. 
Fix a $Y\in\mathcal{Y}$ with $\sum_{k=1}^r\theta_kd_k(Y)=\xi$ and consider the corresponding $r$ clusters as well as the set of isolated nodes. Notice that for any $Y'\in[Y]$ we also have $\sum_{k=1}^r\theta_kd_k(Y')=\xi\,$. In the following, we will construct an ordering for the clusters of $Y$ according to $Y^\st\,$. Denote the clusters of $Y^\st$ by $V_1^\st,\ldots,V_r^\st,$ and $V_{r+1}^\st\,$. 

%Without loss of generality, assume that the cluster of $Y^\st$ are labeled in a way that the sizes are non-increasing; i.e., $n_1\geq n_2\geq \ldots \geq n_r\,$. 
Consider the set of values of cluster sizes $\{n_1,\ldots,n_r\} = \{\eta_1,\ldots, \eta_s\}$ where $\eta_1,\ldots, \eta_{s}$ are distinct, and define $\mathcal{I}_\ell = \{k:\; n_k =\eta_\ell\}\subset\{1,\ldots,r\}$ for $\ell=1,\ldots, s\,$. 
For any $\eta_\ell$ of multiplicity 1 (i.e., $\abs{\mathcal{I}_\ell}=1$), the cluster in $Y\in\mathcal{Y}$ of size $\eta_\ell$ can be uniquely assigned to a cluster among $V_1^\st,\ldots,V_r^\st$ of similar size. 
We now define an ordering for the remaining clusters. Consider a $\eta_\ell$ of multiplicity larger than 1, and restrict the attention to clusters $V$ of size $\eta_\ell$ and clusters $V_k^\st$ for $k\in\mathcal{I}$ (all clusters in $Y^\st$ of size $\eta_\ell$). This is similar to the case in \cite{chen2014statistical} where all sizes are equal: 
For each new cluster $V$ of size $\eta_\ell$, if there exists a $k\in\mathcal{I}_\ell$ such that $\abs{V\cap V_k^\st}>\tfrac{1}{2}\eta_\ell$ then we label this cluster as $V_k\,$; this label is unique. The remaining unlabeled clusters are labeled arbitrarily by a number in $\mathcal{I}_\ell\,$. 

Hence, we labeled all the clusters of $Y$ according to the clusters of $Y^\st\,$. 
For each $(k, k') \in \{1,\ldots,r\} \times \{1,\ldots,r+1\}$, we use $\alpha_{kk'} := \abs{V_k^\st \cap V_{k'}}$ to denote the sizes of intersections of the true and new clusters. 
We observe that the new clusters $(V_1,\ldots,V_{r+1})$ have the following properties:
\begin{enumerate}[({A}1)]
\item $(V_1,\ldots,V_{r+1})$ is a partition of $\{1,\ldots,n\}$ with $\abs{V_k}=n_k$ for all $k=1,\ldots,r\,$; since $Y\in\mathcal{Y}\,$. 
\item For $\ell\in\{1,\ldots,s\}$ with $\abs{\mathcal{I}_\ell}=1\,$, we have $\alpha_{kk} = n_k$ for the index $k\in\mathcal{I}_\ell\,$. 
\item For $\ell\in\{1,\ldots,s\}$ with $\abs{\mathcal{I}_\ell}>1\,$, consider any $k\in\mathcal{I}_\ell\,$. Then, exactly one of the following is true: 
(1) $\alpha_{kk}>\tfrac{1}{2}n_k$; 
(2) $\alpha_{kk'}\leq \tfrac{1}{2}n_k$ for all $k'\in\mathcal{I}_\ell\,$. 

\item For $d_k(Y) = \iprod{Y^\st_{\Cc_k}}{Y^\st - Y}\,$, where $k=1\,\ldots,r\,$, we have
\begin{align*}
d_k(Y)
&= \abs{\{ (i,j):\; (i,j)\in\Cc_k^\st\,,\, Y_{ij} = 0 \}} \\
&= \abs{\{ (i,j):\; (i,j)\in\Cc_k^\st\,,\, (i,j)\in \Cc_{r+1} \}}
  + \sum_{k'\neq k''} \abs{\{ (i,j):\; (i,j)\in\Cc_k^\st\,,\, (i,j)\in V_{k'}\times V_{k''} \}} \\
&= \alpha_{k(r+1)}^2 + \sum_{k'\neq k''} \alpha_{kk'}\alpha_{kk''}  \,,
\end{align*} 
which implies 
\begin{align*}
\xi = \sum_{k=1}^r \theta_k d_k(Y) 
= \sum_{k=1}^r \theta_k \alpha_{k(r+1)}^2
+ \sum_{k=1}^r \sum_{k'\neq k''} \theta_k\alpha_{kk'}\alpha_{kk''}  \,.
\end{align*}

Unless specified otherwise, all the summations involving $k'$ or $k''$ are over the randge $1,\ldots,r+1\,$. 
\end{enumerate}
We showed that the ordered partition for a $Y\in\mathcal{Y}$ with $\sum_{k=1}^r\theta_kd_k(Y)=\xi$ satisfies the above properties. Therefore, 
\[
\abs{\{ [Y]\in\mathcal{Y}:\;  \sum_{k=1}^r \theta_k d_k(Y)  = \xi \}} \leq
\abs{\{ (V_1,\ldots, V_{r+1}) \text{ satisfying the above conditions}  \}} \,.
\]
Next, we upper bound the right hand side of the above. 

Fix an ordered clustering $(V_1,\ldots, V_{r+1})$ which satisfies the above conditions. Define, 
\[
m_1 :=  \sum_{k'\neq 1} \alpha_{1k'}
\]
as the number of nodes in $V_1^\st$ that are misclassified by $Y\,$; hence $m_1+\alpha_{11}=n_1\,$. Consider the following two cases:
\begin{itemize}
\item if $\alpha_{11}>n_1/4$ we have
\[
\sum_{k'\neq k''} \alpha_{1k'}\alpha_{1k''} 
\geq \alpha_{11} \sum_{k''\neq 1} \alpha_{1k''}
> \tfrac{1}{4}n_1m_1
\]
\item if $\alpha_{11}\leq n_1/4$ we have $m_1\geq 3n_1/4\,$, which from the aforementioned properties, we must have $\alpha_{1k'}\leq n_1/2$ for all $k'=1,\ldots, r\,$. Then, 
\[
\sum_{k'\neq k''} \alpha_{1k'}\alpha_{1k''} + \alpha_{1(r+1)}^2 
\geq \sum_{1\neq k'\neq k''\neq1} \alpha_{1k'}\alpha_{1k''} +  \alpha_{1(r+1)}^2
= m_1^2 - \sum_{k'=2}^r \alpha_{1k'}^2 
\geq  m_1^2 - \tfrac{1}{2}n_1m_1 
\geq \tfrac{1}{4}n_1m_1
\]
\end{itemize}
Therefore, 
\[
d_1(Y) = \sum_{k'\neq k''} \alpha_{1k'}\alpha_{1k''} + \alpha_{1(r+1)}^2 \geq \tfrac{1}{4} n_1m_1
\]
which as well holds for other indices besides $k=1\,$. This yields 
\[
%m_k \leq \frac{4d_k(Y)}{n_k} \quad, \quad 
%\sum_{k=1}^r m_k \leq 4\sum_{k=1}^r \frac{d_k(Y)}{n_k} \quad, \quad 
\xi \geq \tfrac{1}{4}\sum_{k=1}^r \theta_kn_km_k \geq \tfrac{1}{4}(\min_k\, \theta_kn_k)\sum_{k=1}^r m_k \quad \implies \quad 
\bar w := \sum_{k=1}^r m_k \leq \frac{4\xi}{\min_k\, \theta_kn_k} :=M
\]
where $\bar w$ is the number of misclassified non-isolated nodes. Since, one misclassified isolated node produces one misclassified non-isolated node, we have $w_0\leq \bar w\leq M$ where $w_0$ is the number of misclassified isolated nodes.

\begin{itemize}
\item The pair of numbers $(\bar w,w_0)$ can take at most $M^2$ different values. 
\item For each such pair of counts, there are at most ${\bar n}^{2M}$ ways to choose the identity of the misclassified nodes. 
\item Each misclassified non-isolated node can be assigned to one of $r-1\leq \bar n$ different clusters or be left isolated, and each misclassified isolated node can be assigned to one of $r\leq \bar n$ clusters. 
\end{itemize}
All in all, 
\[
\abs{\{ [Y]\in\mathcal{Y}:\;  \sum_{k=1}^r \theta_k d_k(Y)  = \xi \}} 
\leq M^2 {\bar n}^{4M} =  \left(\frac{4\xi}{\min_k\, \theta_kn_k}\right)^2 \exp\left( \frac{16\xi}{\min_k\, \theta_kn_k} \log \bar n \right) \,.
\]
\end{proof}

\subsection{Proofs for Non-recoverability}
\begin{proof}[of cases \ref{condn:impossible-first} and \ref{condn:impossible-second} of Theorem \ref{thm:impossibility}]
Let $\prob_{(Y^\st,A)}$ be the joint distribution of $Y^\st$ and $A$, where $Y^\st$ is sampled uniformly from $\Yc$ and $A$ is generated according to the heterogenous stochastic block model conditioning on $Y^\st$. Note that
\[
\inf_{\hat{Y}}\sup_{Y^\st \in \Yc}\prob[\hat{Y}\neq Y^\st]\geq \inf_{\hat{Y}}\prob_{(Y^\st,A)}[\hat{Y}\neq Y^\st]\,.
\]
By Fano's inequality we have,
\beqa\label{eqn:Fano}
\prob_{(Y^\st,A)}[\hat{Y}\neq Y^\st]\geq 1-\frac{I(Y^\st;A)+1}{\log\abs{\Yc}},
\eeqa
where $I(X;Z)$ is the mutual information, and $H(X)$ is the Shannon entropy for $X$. By counting argument we find that $\abs{\Yc}=\binom{n}{\bar n }\frac{\bar n !}{n_1!\ldots n_r!}$. Using $\sqrt{n}(n/e)^n\leq n!\leq e\sqrt{n}(n/e)^n$ and $\binom{n}{\bar n }\geq (n/\bar n )^{\bar n }$, it follows that
\[
\abs{\Yc}\geq \frac{n^{\bar n }\sqrt{\bar n }}{e^r\sqrt{n_1\ldots n_r}n_1^{n_1}\ldots n_r^{n_r}}
\]
which gives
\[
\log \abs{\Yc}\geq \sum_{i=1}^r n_i\big(\log{n \over n_i}-\frac{\log n_i}{2n_i}\big)-r\geq{1 \over 2}\sum_{i=1}^r n_i\log{n \over n_i}-r\,. % the inequality here uses n_i <= n/2
\]
On the other hand, note that $H(A)\leq \binom{n}{2}H(A_{12})$ by chain rule, the fact that $H(X|Y)\leq H(X)$, and the symmetry among identically distributed $A_{ij}$'s. Furthermore $A_{ij}$'s are conditionally independent and hence $H(A|Y^\st)=\binom{n}{2}H(A_{12}|Y^\st_{12})$. Now it follows that
$$I(Y^\st;A)=H(A)-H(A|Y^\st)\leq \binom{n}{2}I(Y^\st_{12};A_{12}).$$
Observe that
\[
\prob(Y_{12}^\st=1,(1,2)\in \Cc_i)
=\frac{\binom{n-2}{n_i-2}
\binom{n-n_i}{n_1,\ldots,n_{i-1},n_{i+1},\ldots,n_r,n_0}}{\abs{\Yc}}=\frac{n_i(n_i-1)}{n(n-1)}:=\alpha_i \,.
\]
Using the properties of KL-divergence, we have $\prob(A_{12}=1)=\sum_{i=1}^r \alpha_i   p_i+(1-\sum_{i}\alpha_i)  q:=\beta\,$. Therefore, %\red{by the same argument as in page 19 of \cite{chen2014statistical}}, we get
\beqa \label{eqn:mutualinfo}
I(Y^\st_{12},A_{12})=\sum_{i=1}^r\alpha_i \kldiv(  p_i,\beta)+(1-\sum_{i}\alpha_i)\kldiv(  q,\beta)
{= H(\beta) - \sum \alpha_i H(p_i) - (1-\sum \alpha_i) H(q)}
\eeqa
Since $I(Y^\st;A)\leq \binom{n}{2}I(Y_{12}^\st;A_{12})$, plugging in the following condition in Fano's inequality \eqref{eqn:Fano}, 
\beqa	\label{eq:imp-bound}
\big( {\tfrac{1}{2}\sum_i n_i\log \frac{n}{n_i}-r} \big) 
\geq 2+2\binom{n}{2} I(Y^\st_{12};A_{12})  \,,  %or 4\vee 4 instead of 2+2
\eeqa
guarantees $\prob_{(Y^\st,A)}(\hat{Y}\neq Y^\st)\geq \tfrac{1}{2}$. In the following, we bound $I(Y^\st_{12};A_{12})$ in two different ways to derive conditions \ref{condn:impossible-first} and \ref{condn:impossible-second} of Theorem \ref{thm:impossibility}. 
Throughout the proof we use the following inequality from \cite{chen2014statistical} for the Kullback-Leibler divergence of Bernoulli variables,
\beqa	\label{eqn:KL-ineq}
\kldiv(p,q):=\kldiv(\ber(p),\ber(q))=p \log \frac{p}{q}+(1-p)\log\frac{1-p}{1-q}\leq \frac{(p-q)^2}{q(1-q)}\,,
\eeqa
where the inequality is established by $\log x \leq x-1\,$, for any $x\geq 0\,$. 

\bit
\item From \eqref{eqn:mutualinfo}, we have 
\beqa 
I(Y^\st_{12},A_{12})
\leq \sum_{i=1}^r \frac{4\alpha_i (p_i-q)^2}{  q(1-  q)} 
\leq \frac{ 4 \sum_{i=1}^r n_i^2(p_i-q)^2}{n(n-1) q(1-  q)} 
\eeqa
where we assumed $\sum n_i^2 \leq \tfrac{1}{2}n^2\,$.  
Now, the right hand side of \ref{eq:imp-bound} can be bounded as
\[
2\binom{n}{2} I(Y^\st_{12};A_{12}) 
\leq \frac{4\sum_{i=1}^r n_i^2(p_i-q)^2}{q(1-  q)} 
=4\sum_{i=1}^r n_i^2\chidiv(p_i,q)
\]
and gives the sufficient condition \ref{condn:impossible-first} of Theorem \ref{thm:impossibility}. 
%Plugging in \eqref{eqn:mutualinfo} into Fano's inequality in \eqref{eqn:Fano} yields
%\beqas
%\prob_{(Y^\st,A)}(\hat{Y}\neq Y^\st)
%&\geq 1-\frac{1+ \binom{n}{2} \frac{\sum_{i=1}^r n_i^2(p_i-q)^2}{n(n-1) q(1-  q)}  }{\tfrac{1}{2}\sum_i n_i\log(n/n_i)-r}
%\\
%&\geq 1-\frac{1+ \frac{  \sum_{i=1}^r n_i^2(p_i-q)^2}{2 q(1-  q)}  }{\tfrac{1}{2}\sum_i n_i\log(n/n_i)-r}
%\\
%&\geq {3 \over 4}-\frac{  \sum_{i=1}^r n_i^2(p_i-q)^2}{2 q(1-  q)\left({1\over 2}\sum_i n_i\log(n/n_i)-r\right)}\\
%&\geq {1 \over 2}\,,
%\eeqas
%where the third inequality follows from the assumption $\tfrac{1}{2}\sum_i n_i\log(n/n_i)-r\geq {1 \over 2}\sum_i n_i-r\geq 4$, and the last from condition \ref{condn:impossible-first}. 

\item 
%Next we turn to prove the conclusion by the condition \ref{condn:impossible-second}. 
Again from (\ref{eqn:mutualinfo}), we have
\beqas
I(Y_{12}^\st;A_{12}) & = \sum_i \alpha_i \bigg(  p_i \log \frac{  p_i}{\beta}+\big(1-  p_i\big)\log\frac{1-  p_i}{1-\beta}\bigg)+\big(1-\sum_i \alpha_i\big)\kldiv(  q,\beta)\\
&\leq \sum \alpha_i   p_i \log \frac{1}{\alpha_i} +\log c + \big(1-\sum_i \alpha_i\big)\frac{(  q-\beta)^2}{\beta(1-\beta)} 
\eeqas
where the first term is bounded via $\beta\geq \sum_i \alpha_i  p_i\geq \alpha_i  p_i\,$, the second term is bounded via $\beta\leq p_{\max}$ and $c=({1-  p_{\min}})/({1-  p_{\max}})\,$, and we used \eqref{eqn:KL-ineq} for the last term. 
Since $1-\beta=1-  q-\sum_i \alpha_i (p_i-q)\geq (1-\sum_i\alpha_i)(1-  q)\,$, the last term can be bounded as
\beqas
\big(1-\sum_i \alpha_i\big)\frac{(  q-\beta)^2}{\beta(1-\beta)}&\leq \big(1-\sum_i \alpha_i\big)\frac{\big(\sum_i \alpha_i (p_i-q)\big)^2}{\big(\sum_i \alpha_i  p_i\big)\big(1-\sum_i\alpha_i\big)(1-  q)}\leq   \sum_i\alpha_i (p_i-q)\leq \sum_i\alpha_i p_i \,.
\eeqas
This implies
\beqa
I(Y_{12}^\st;A_{12}) & \leq  \sum_i \alpha_i p_i\log\tfrac{1}{\alpha_i}+ \sum_i\alpha_i p_i+\log c
\leq \sum_i \alpha_i p_i\log\tfrac{e}{\alpha_i}+\log c.\\
\eeqa
Since $n_i\geq 2$, $\alpha_i=\frac{n_i(n_i-1)}{n(n-1)}\geq \tfrac{n_i^2}{en^2}$. Hence
\beqas
2\binom{n}{2}I(Y_{12}^\st;A_{12}) 
\leq n(n-1) \sum_i\frac{n_i(n_i-1)}{n(n-1)}p_i\log\frac{e^2n^2}{n_i^2}+2\log c
\leq   2\sum_i n_i^2p_i\log \frac{e n}{n_i} +2\log c
\eeqas
which gives the sufficient condition \ref{condn:impossible-second} of Theorem \ref{thm:impossibility}. 
\eit
\end{proof}

%%%%%%%%
\begin{proof}[of case \ref{condn:impossible-third} in Theorem \ref{thm:impossibility}]
Without loss of generality assume $n_1\leq n_2\leq \ldots\leq n_r\,$. Let $M:=\bar n-n_{\min}=\bar n-n_1\,$, and $\bar \Yc := \{Y_0,Y_1,\ldots,Y_M\}\,$. $Y_0$ is the clustering matrix with clusters $\{\Cc_\ell\}_{\ell=1}^r$ that correspond to $V_1=\{1,\ldots,n_1\}\,$, $V_\ell =\{\sum_{i=1}^{\ell-1} n_i+1,\ldots,\sum_{i=1}^\ell n_i\}$ for $\ell= 2,\ldots,r\,$. Other members of $\bar\Yc$ are given by swapping an element of $\cup_{\ell=2}^r V_\ell$ with an element of $V_1\,$. 
Let $\prob_i$ be the distributional law of the graph $A$ conditioned on $Y^\st=Y_i\,$. Since $\prob_i$ is product of ${1 \over 2}n(n-1)$ Bernoulli random variables, we have
\beqa
I(Y^\st;A)&=\E_Y\left[\kldiv\left(\prob(A|Y),\prob(A)\right)\right]\\
&=\tfrac{1}{M+1}\sum_{i=0}^M \kldiv\big(\prob_i ,\tfrac{1}{M+1}\sum_{j=0}^M \prob_{j}\big)\\
&\leq \tfrac{1}{(M+1)^2}\sum_{ i,j=0}^M\kldiv(\prob_i,\prob_{j})\\
&\leq \max_{i,j=0,\ldots,M} \;\kldiv(\prob_i,\prob_{j})\\
&\leq \max_{i_1,i_2,i_3=1,\ldots,r}\; \sum_{j=1}^3
\left( \frac{n_{i_j}(p_{i_j}-q)^2}{q(1-q)} + \frac{n_{i_j}(p_{i_j}-q)^2}{p_{i_j}(1-p_{i_j})}  \right) \\
%&\leq 6\max_{i=1,\ldots,r}  \frac{n_i(p_i-q)^2}{\min\{q(1-q),\, p_i(1-p_i)\}}
&\leq 3 \max_{i=1,\ldots,r}\; 
\left( \frac{n_{i}(p_{i}-q)^2}{q(1-q)} + \frac{n_{i}(p_{i}-q)^2}{p_{i}(1-p_{i})}  \right) 
\eeqa
where the third line follows from the convexity of KL-divergence, and the line before the last follows from the construction of $\bar\Yc$ and \eqref{eqn:KL-ineq}. 
Now if the condition of the theorem holds, then $I(Y^\st;A)\leq {1 \over 4}\log(n-n_{\min})={1 \over 4}\log\abs{\bar\Yc}$. 
Note that for $n\geq 128$ we get $\log\abs{\bar\Yc}=\log(n-n_{\min})\geq \log(n/2)\geq 4\,$. The conclusion follows by Fano's inequality in \eqref{eqn:Fano} restricting the supremum to be taken over $\bar \Yc\,$.
\end{proof}

\section{Recovery by a Simple Counting Algorithm}	\label{sec:simple}
In Section \ref{sec:convex}, we considered a tractable approach for exact recovery of (partially) observed models generated according to the heterogenous stochastic block model. However, in the interest of computational effort, one can further characterize a subset of models that are recoverable via a much simpler method than the convex program. The following algorithm is a proposal to do so. Moreover, the next theorem provides a characterization for models for which this simple thresholding algorithm is effective for exact recovery. Here, we allow for isolated nodes as described in Section \ref{sec:main-results}.

\begin{algorithm}[H]
\begin{algorithmic}[1]
\STATE (Find isolated notes) For each node $v$, compute its degree $d_v$. Declare $i$ as isolated if
$$d_v <\min_k \frac{(n_k-1)  (p_k-q)}{2}+(n-1)  q.$$
\STATE (Find all communities) For every pair of nodes $(v,u)$, compute the number of common neighbors $S_{vu}:=\sum_{w\neq v,u}A_{vw}A_{uw}$. Declare $v,u$ as in the same community if
%\[
%S_{vu}>{1 \over 2}\bigg(\min_k \big((n_k-2) p_k^2+(n-n_k) q^2\big)+\max_{i\neq j}\big((n_k-1) p_kq+(n_l-1) p_lq+(n-n_k-n_l) q^2\big)\bigg).
%\]
\[
S_{vu}>nq^2 + {1 \over 2}\left(\min_k \left((n_k-2) p_k^2-n_k q^2\right)+q\cdot \max_{i\neq j}\left(\rho_k- p_k+ \rho_l-p_l\right)\right)
\]
where $\rho_k = n_k(p_k-q)\,$. 
\end{algorithmic}
\caption{{\sc Simple Thresholding Algorithm}}
\label{alg:simplethresholding}
\end{algorithm}

\begin{theorem}	\label{thm:simple-recovery}
Under the stochastic block model, with probability at least $1-2n^{-1}$, the simple counting algorithm \ref{alg:simplethresholding} find the isolated nodes provided
\begin{equation}\label{simpiso}
\min_k \, (n_k-1)^2 (p_k-q)^2\geq 19(1-q)\left(\max_k\, n_k  p_k+n q \right) \log n \,.
\end{equation}
Furthermore the algorithm finds the cluster if %$q\leq 1/\sqrt{2}$ and
\beqa\label{simpfindcluster}
 &\left[\min_k \left\{(n_k-2)p_k^2+(n-n_k)q^2\right\}   -q\, \max_{k\neq l}\left\{ (n_k-1) p_k+(n_l-1) p_l+(n-n_k-n_l) q\right\}\right]^2\\
\geq &  \quad 26 (1- q^2) \left(\max_k\, n_k p_k^2+n q^2\right) \log n \,,
\eeqa
while the term inside the bracket (which is squared) is assumed to be non-negative. 
\end{theorem}
We remark that the following is a slightly more restrictive condition than \eqref{simpfindcluster} 
\beqa
\left[\min_k n_k(p_k^2-q^2) -2q \rho_{\max} \right]^2
\geq  26(1-q^2) \left[n q^2 +\max_k\, n_k p_k^2\right] \log n \,.
\eeqa 
with better interpretability. 
%Moreover, \eqref{simpiso} can be replaced by the following with slightly different constants
%    \begin{equation}
%    	\rho_{\min}^2 \geq 19 (1-q) \left(\max_k n_k p_k+n q \right) \log n \,.
%    \end{equation}

\begin{proof}[of Theorem \ref{thm:simple-recovery}]
For node $v$, let $d_v$ denote its degree. Let $\bar V=\cup_{i=1}^r V_i$ denote the set of nodes which belong to one of the clusters, and $V_0$ be isolated nodes. 
If $v \in V_i$ for some $i=1,\ldots,r\,$, then $d_v$ is distributed as a sum of independent binomial random variables $\bin(n_i-1,  p_i)$ and $\bin(n-n_i,  q)\,$. If $v \in V_0\,$, then $d_v$ is distributed as $\bin(n-1,  q)\,$. Hence we have,
$$\E[d_v]=
\begin{cases}
(n_i-1)  p_i+(n-n_i)  q & v \in V_i\subset \bar V\\
(n-1)  q & v \in V_0 \,,\\
\end{cases}
$$
and
$$\mathrm{Var}[d_v]=
\begin{cases}
(n_i-1)  p_i(1-  p_i)+(n-n_i)  q(1-  q) & v \in V_i\subset \bar V\\
(n-1)  q(1-  q) & v \in V_0 \,.\\
\end{cases}
$$
Let $\kappa_0^2:=\max_i n_i   p_i(1-  q)+n  q(1-  q)$, and $t=\min_i \frac{(n_i-1) (p_i-q)}{2}\leq \frac{\kappa_0^2}{2}$. Then $\mathrm{Var}[d_v]\leq \kappa_0^2$ for any $v\in V_0\cup \bar V\,$. By Bernstein's inequality we get
\beqa
\prob\big[\abs{d_v-\E[d_v]} > t\big]&\leq 2\exp\bigg(-\frac{t^2}{2\kappa_0^2+2t/3}\bigg)\leq 2\exp\bigg(-\frac{3\min_i (n_i-1)^2 (p_i-q)^2}{28\kappa_0^2}\bigg)\leq 2n^{-2},
\eeqa
where the last inequality follows from the condition (\ref{simpiso}). Now by union bound over all nodes, with probability at least $1-2n^{-1}$, for node $v \in V_i \subset \bar V$ we have,
\begin{align}
d_v\geq (n_i-1)  p_i+(n-n_i)  q-t>\min_i \frac{(n_i-1)  (p_i-q)}{2}+(n-1)  q\,,
\end{align}
and for node $v \in V_0\,$,
\begin{align}
d_v\leq  (n-1)  q(1-  q)+t<\min_i \frac{(n_i-1)  (p_i-q)}{2}+(n-1)  q\,.
\end{align}
This proves the first statement of the theorem, and all the isolated nodes are correctly identified.  For the second statement, let $S_{vu}$ denote the common neighbor for nodes $v,u \in \bar V$. Then
\beqas
S_{vu}\sim_d
\begin{cases}
\bin(n_i-2, p_i^2)+\bin(n-n_i, q^2) & (v,u)\in V_i\times V_i\\
\bin(n_i-1, p_iq)+\bin(n_j-1, p_jq)+\bin(n-n_i-n_j, q^2) & (v,u)\in V_i\times V_j \,,\; i\neq j \\
\end{cases}
\eeqas
where $\sim_d$ denotes equality in distribution and $+$ denotes the summation of independent random variables. Hence
$$\E[S_{vu}]=
\begin{cases}
(n_i-2) p_i^2+(n-n_i) q^2 & (v,u) \in V_i\times V_i \\
(n_i-1) p_iq+(n_j-1) p_jq+(n-n_i-n_j) q^2 & (v,u) \in V_i\times V_j\,,\; i\neq j \\
\end{cases}
$$
and
$$\mathrm{Var}[S_{vu}]=
\begin{cases}
(n_i-2) p_i^2(1- p_i^2)+(n-n_i) q^2(1- q^2) & (v,u) \in V_i\times V_i\\
(n_i-1) p_iq(1- p_iq)+(n_j-1) p_jq(1- p_jq)\\
\quad\quad+(n-n_i-n_j) q^2(1- q^2) & (v,u) \in V_i\times V_j\,,\; i\neq j \\
\end{cases}
$$
%Let 
%\begin{align*}
%\Delta
%&=\min_i \big((n_i-2) p_i^2+(n-n_i) q^2\big)-\max_{i\neq j}\big((n_i-1) p_iq+(n_j-1) p_jq+(n-n_i-n_j) q^2\big) \\
%&=\min_i \big((n_i-2) p_i^2  -n_i q^2\big)-\max_{i\neq j}\big((n_i-1) p_iq+(n_j-1) p_jq - (n_i+n_j) q^2\big) \,, 
%\end{align*}
%and $t=\Delta/2$. 
%Let $\kappa_1^2:=\max_i n_i p_i^2(1- q^2)+n q^2(1- q^2)$. Then $\mathrm{Var}[S_{vu}]\leq \kappa_1^2$ for all $v,u\,$.  
%It is easy to verify that $t\leq \kappa_1^2/2$ given $q\leq 1/\sqrt{2}\,$. Bernstein's inequality yields
%\beqa
%\prob\big[\abs{S_{vu}-\E[S_{vu}]}>t\big]&\leq 2\exp\bigg(-\frac{t^2}{2\kappa_1^2+2t/3}\bigg)= 2\exp\bigg(-\frac{3\Delta^2}{28\kappa_1^2}\bigg)\leq 2n^{-3},
%\eeqa
%where the last line follows from assumption (\ref{simpfindcluster}). By union bound over all pair of nodes $(v,u)$, we get with probability at least $1-2n^{-1}$, $S_{vu}> \Gamma$ for all $v,u$ in the same cluster and $S_{vu}<\Gamma$ otherwise. Here
%$$\Gamma:={1 \over 2}\bigg(\min_i \big((n_i-2) p_i^2+(n-n_i) q^2\big)+\max_{i\neq j}\big((n_i-1) p_iq+(n_j-1) p_jq+(n-n_i-n_j) q^2\big)\bigg).$$
%
Let 
\begin{align*}
\Delta
&=\min_i \big((n_i-2) p_i^2+(n-n_i) q^2\big)-\max_{j}\big(2(n_j-1) p_jq+(n-2n_j) q^2\big) \\
&=\min_i \big((n_i-2) p_i^2  -n_i q^2\big)-\max_{j}\big(2(n_j-1) p_jq - 2n_j q^2\big) \,, %\red{\text{could be neg, hence }t}
\end{align*}
Let $\kappa_1^2:=2\max_i n_i p_i^2(1- q^2)+n q^2(1- q^2)$. Then $\mathrm{Var}[S_{vu}]\leq \kappa_1^2$ for all $v,u\,$.  Then $\Delta \leq \kappa_1^2/2\,$. 
Bernstein's inequality with $t=\Delta/2$ yields
\beqa
\prob\big[\abs{S_{vu}-\E[S_{vu}]}>t\big]&\leq 2\exp\bigg(-\frac{t^2}{2\kappa_1^2+2t/3}\bigg)\leq  2\exp\bigg(-\frac{3\Delta^2}{26\kappa_1^2}\bigg)\leq 2n^{-3},
\eeqa
where the last line follows from assumption (\ref{simpfindcluster}). By union bound over all pair of nodes $(v,u)$, we get with probability at least $1-2n^{-1}$, $S_{vu}> \Gamma$ for all $v,u$ in the same cluster and $S_{vu}<\Gamma$ otherwise. Here
$$\Gamma:={1 \over 2}\bigg(\min_i \big((n_i-2) p_i^2+(n-n_i) q^2\big)+\max_{i\neq j}\big((n_i-1) p_iq+(n_j-1) p_jq+(n-n_i-n_j) q^2\big)\bigg).$$

\color{black}
\end{proof}

\section{Detailed Computations for Examples in Section \ref{sec:this-paper}} \label{app:verification}

In the following, we present the detailed computations for the examples in Section \ref{sec:this-paper} and summarized in Table \ref{tab:examples}. When there is no impact on the final result, quantities are approximated as denoted by $\approx\,$. 

First, we repeat the conditions of Theorems \ref{thm:convex_recovery} and \ref{thm:convex_recovery2}. 
The conditions of Theorem \ref{thm:convex_recovery} can be equivalently stated as 
\begin{itemize}
\item $\rho_k^2 \gtrsim     n_kp_k(1-p_k)\log n_k = \sigma_k^2 \log n_k $
\item $(p_{\min}-q)^2 \gtrsim     q(1-q)\tfrac{\log n_{\min}}{n_{\min}}$
\item $\rho_{\min}^2 \gtrsim     \max \left\{ \log n, nq(1-q), \max_k \, n_kp_k(1-p_k) \right\}$
\item  $\sum_{k=1}^r n_k^{-\alpha}  =  o(1)$ for some $\alpha>0\,$. 
\end{itemize}
Notice that $n_kp_k(1-p_k)\gtrsim \log n_k\,$, for $k=1,\ldots,r\,$, is implied by the first condition, as mentioned in Remark \ref{rem:connected-convex}. 
The conditions of Theorem \ref{thm:convex_recovery2} can be equivalently stated as 
\begin{itemize}
\item $\rho_k^2 \gtrsim     n_kp_k(1-p_k)\log n$
\item $(p_{\min}-q)^2 \gtrsim     q(1-q)\tfrac{\log n}{n_{\min}}$
\item $\rho_{\min}^2 \gtrsim     \max \left\{ nq(1-q), \max_k \, n_kp_k(1-p_k) \right\}$.
\end{itemize}

\begin{remark}\label{rem:pq-far}
Provided that both $p_k$ and $q/p_k$ are bounded away from $1\,$, we have 
\begin{align}
\chidiv(q,p_k) = p_k \frac{(1-q/p_k)^2}{1-p_k} \approx p_k \quad,\quad
\frac{\rho_k^2}{\sigma_k^2} = \frac{(1-q/p_k)^2}{1-p_k}\, n_kp_k \approx n_kp_k \,.
\end{align}
This simplifies the first condition of Theorem \ref{thm:convex_recovery} to a simple connectivity requirement. Hence, we can rewrite the conditions of Theorems \ref{thm:convex_recovery}, \ref{thm:convex_recovery2} as
\begin{align*}
\ref{thm:convex_recovery}& : 
    n_kp_k\gtrsim \log n_k\,, 
    \chidiv(p_{\min},q)\gtrsim \tfrac{\log n_{\min}}{n_{\min}}\,, 
    \rho_{\min}^2 \gtrsim     \max \left\{ \sigma_{\max}^2, nq(1-q),\log n \right\}, 
    \sum_{k=1}^r n_k^{-\alpha}  =  o(1) \text{ for some } \alpha>0 \\
\ref{thm:convex_recovery2} &: 
    n_kp_k\gtrsim \log n\,, 
    \chidiv(p_{\min},q)\gtrsim \tfrac{\log n}{n_{\min}}\,,     
    \rho_{\min}^2 \gtrsim     \max \left\{ \sigma_{\max}^2, nq(1-q)\right\}\,. 
\end{align*}
\end{remark}

%---------------------------------------------------------------------------------------------------------
\paragraph{Example \ref{ex:we-can1}:} 
In a configuration with two communities $\tri{n-\sqrt{n}}{n^{-2/3}}{1}$ and $\tri{\sqrt{n}}{\tfrac{1}{\log n}}{1}$ with $q=n^{-2/3-0.01}\,$, we have $n_{\min} = \sqrt{n}$ and $p_{\min} = n^{-2/3}\,$. 
We have, 
\[
\chidiv(p_{\min},q) \approx n^{-2/3+0.01}
\]
which does not exceed either $\tfrac{\log n_{\min}}{n_{\min}} \approx \tfrac{\log n}{\sqrt{n}}$ or $\tfrac{\log n}{n_{\min}} \approx \tfrac{\log n}{\sqrt{n}}\,$, 
and we get no recovery guarantee from Theorems \ref{thm:convex_recovery} and \ref{thm:convex_recovery2} respectively. 
However, as $p_{\min}-q$ is not much smaller than $q\,$, while $\rho_{\min}\approx n^{1/3}$ grows much faster than $\log n\,$, the condition of Theorem \ref{thm:hard-recovery} trivially holds. 

Here are the related quantities for this configuration: 
\[
\rho_1 = n_1(p_1-q) = (n-\sqrt{n})(n^{-2/3} - n^{-2/3-0.01}) \approx n^{1/3} \quad , \quad 
\rho_2 = n_2(p_2-q) = \sqrt{n}(\tfrac{1}{\log n} - n^{-2/3-0.01}) \approx \tfrac{\sqrt{n}}{\log n}
\]
which gives $\rho_{\min}\approx n^{1/3}\,$. 
Furthermore, 
\[
\sigma_1^2 = n_1p_1(1-p_1) \approx n^{1/3} 
\quad , \quad 
\sigma_2^2 = n_2p_2(1-p_2) = \tfrac{\sqrt{n}}{\log n}\,, 
\]
which gives $\sigma_{\max} = \tfrac{\sqrt{n}}{\log n}\,$. On the other hand $nq(1-q) \approx n^{1/3-0.01}$ which is smaller than $\sigma_{\max}^2\,$. 

%---------------------------------------------------------------------------------------------------------
\paragraph{Example \ref{ex:we-can2}:} 
Consider a configurations with $\tri{n-n^{2/3}}{n^{-1/3+\epsilon}}{1}$ and $\tri{\sqrt{n}}{\tfrac{c}{\log n}}{n^{1/6}}$ and $q=n^{-2/3+3\epsilon}\,$.  
Since all $p_k$'s and $q/p_k$'s are much less than $1\,$, the first condition of both Theorems \ref{thm:convex_recovery} and \ref{thm:convex_recovery2} can be verified by Remark \ref{rem:pq-far}. Moreover, $n_{\min} = \sqrt{n}$ and $p_{\min} =n^{-1/3+\epsilon}$ which gives 
\[
\chidiv(p_{\min},q) = n^{-\epsilon} 
\]
and verifies $\chidiv(p_{\min},q) \gtrsim \tfrac{\log n_{\min}}{n_{\min}}$ for \ref{thm:convex_recovery}, as well as $\chidiv(p_{\min},q) \gtrsim \tfrac{\log n}{n_{\min}}$ for \ref{thm:convex_recovery2}. 
Moreover, $\rho_1 \approx n^{2/3+\epsilon}$ and $\rho_2 \approx \tfrac{\sqrt{n}}{\log n}$ which gives  $\rho_{\min}\approx \tfrac{\sqrt{n}}{\log n} \gtrsim \sqrt{\log n}\,$. 
On the other hand, $\sigma_1^2 \approx n^{2/3+\epsilon}$ and $\sigma_2^2 \approx \sqrt{n}/\log n$ which gives 
\[
\max\{ \sigma_{\max}^2\,, \, nq(1-q)\} \approx n^{2/3+\epsilon}\,.
\]
Thus all conditions of Theorems \ref{thm:convex_recovery} and \ref{thm:convex_recovery2} are satisfied. Moreover, as $p_{\min}-q$ is not much smaller than $q\,$, while $\rho_{\min}\approx \tfrac{\sqrt{n}}{\log n}$ is growing much faster than $\log n\,$, the condition of Theorem \ref{thm:hard-recovery} trivially holds. 

%---------------------------------------------------------------------------------------------------------
\paragraph{Example \ref{ex:cvx-thm1-sqrtlogn}:}

Consider a configurations with $\tri{\sqrt{\log n}}{O(1)}{m}$ and $\tri{n_2}{O(\tfrac{\log n}{\sqrt{n}})}{\sqrt{n}}$ and $q=O(\log n/n)\,$, where $n_2 = \sqrt{n} - m \sqrt{\log n / n}\,$. Here, we assume $m\leq n/(2\sqrt{\log n})$ which implies $n_2 \geq \sqrt{n}/2\,$. 
Since all $p_k$'s and $q/p_k$'s are much less than $1\,$, we can use Remark \ref{rem:pq-far}: 
the first condition of Theorem \ref{thm:convex_recovery} holds as $n_1p_1 \approx \sqrt{\log n} \gtrsim \log n_1\approx \log\log n$ and $n_2p_2\approx \log n \gtrsim \log n_2\,$. 
However, $n_1p_1\approx \sqrt{\log n} \not\gtrsim \log n$ and Theorem \ref{thm:convex_recovery2} does not offer a guarantee for this configuration. 

Moreover, $n_{\min} = \sqrt{\log n}$ and $p_{\min} =O(\tfrac{\log n}{\sqrt{n}})$ which gives 
\[
\chidiv(p_{\min},q) = \log n
\]
and verifies $\chidiv(p_{\min},q) \gtrsim \tfrac{\log n_{\min}}{n_{\min}} \approx \tfrac{\log\log n}{\sqrt{\log n}}$ for \ref{thm:convex_recovery}, as well as $\chidiv(p_{\min},q) \gtrsim \tfrac{\log n}{n_{\min}}= \sqrt{\log n}$ for \ref{thm:convex_recovery2}. 
Moreover, $\sigma_1^2 = \sqrt{\log n}$ (also $\rho_1$) and $\sigma_2^2 = \log n$ (also $\rho_2$) which gives 
\[
\max\{ \sigma_{\max}^2\,, \, nq(1-q)\} \approx \log n
\]
and $\rho_{\min}^2\approx \log n\,$. For the last condition of Theorem \ref{thm:convex_recovery} we need
\[
m (\log n)^{-\alpha/2} + \sqrt{n} (\sqrt{n} - k\sqrt{\tfrac{\log n}{n}})^{-\alpha} = o(1)
\]
for some $\alpha>0$ which can be guaranteed provided that $m$ grows at most polylogarithmically in $n\,$. 
All in all, we verified the conditions of Theorem \ref{thm:convex_recovery} while the first condition of \ref{thm:convex_recovery2} fails. 
Observe that $\rho_{\min}$ fails the condition of Theorem \ref{thm:hard-recovery}. 

Alternatively, consider a configuration with 
$\tri{\sqrt{\log n}}{O(1)}{m}$ and $\tri{\sqrt{n}}{O(\tfrac{\log n}{\sqrt{n}})}{m' }$ and $q = O(\tfrac{\log n}{n}) \,$, where $m' = \sqrt{n} - m \sqrt{\log n / n}$ to ensure a total of $n$ vertices. Here, we assume $m\leq n/(2\sqrt{\log n})$ which implies $m' \geq \sqrt{n}/2\,$. Similarly, all conditions of Theorem \ref{thm:convex_recovery} can be verified provided that $m$ grows at most polylogarithmically in $n\,$. Moreover, the conditions of Theorems \ref{thm:convex_recovery2} and \ref{thm:hard-recovery} fail to satisfy. 

%---------------------------------------------------------------------------------------------------------
\paragraph{Example \ref{ex:cvx-thm2-slogn}:} Consider a configuration with $\tri{\tfrac{1}{2}n^\epsilon}{O(1)}{n^{1-\epsilon}}$ and $\tri{\tfrac{1}{2}n}{n^{-\alpha}\log n}{1}$ and $q = n^{-\beta}\log n\,$, where $0<\alpha<\beta<1$ and $0<\eps<1\,$.  

We have $\rho_1 \approx n^\eps$ and $\rho_2 \approx n^{1-\alpha} \log n\,$. Since $\rho_{\min}^2\gtrsim \log n\,$, the last condition of Theorem \ref{thm:convex_recovery} holds, and $\log n_{\min} \approx \log n\,$, we need to check for similar conditions to be able to use Theorems \ref{thm:convex_recovery} and \ref{thm:convex_recovery2}. Using Remark \ref{rem:pq-far}, the first condition of both Theorems holds because of $n_1p_1 \approx n^\eps \gtrsim \log n$ and $n_2p_2 \approx n^{1-\alpha}\log n \gtrsim \log n\,$. 
Moreover, the condition 
\[
\chidiv(p_{\min}, q) \approx n^{\beta-2\alpha} \log n \gtrsim \tfrac{\log n}{n_{\min}}\approx \tfrac{\log n}{n^\eps}
\]
is equivalent to $\beta+\eps > 2\alpha\,$. Furthermore, $\sigma_1^2 = n^\eps$ and $\sigma_2^2 = n^{1-\alpha} \log n\,$, and for the last condition we need 
\[
\min\{ n^{2\eps}\,,\, n^{2-2\alpha} \log^2 n  \} \gtrsim \max\{ n^{\eps} \,,\, n^{1-\alpha}\log n\,,\, n^{1-\beta}\log n \}
\]
which is equivalent to $2\eps + \alpha >1$ and $\eps+2\alpha <2\,$. Notice that $\beta+1>2\alpha$ is automatically satisfied when we have $\beta+\eps > 2\alpha$ from the previous part. 
%---------------------------------------------------------------------------------------------------------
\paragraph{Example \ref{ex:cvx-thm2-logn}:}
Consider a configuration with $\tri{\log n}{O(1)}{\tfrac{n}{\log n}- m \sqrt{\tfrac{n}{\log n}}}$ and $\tri{\sqrt{n\log n}}{O(\sqrt{\tfrac{\log n}{n}})}{m}$ and $q = O(\tfrac{\log n}{n})\,$. All of $\rho_1\,$, $\rho_2\,$, $\sigma_1^2\,$, $\sigma_2^2\,$, and $nq(1-q)\,$, are approximately equal to $\log n\,$.  Thus, the first and third conditions of Theorems \ref{thm:convex_recovery} and \ref{thm:convex_recovery2} are satisfied. Moreover, 
\[
\chidiv(p_{\min},q)\approx 1 \gtrsim \tfrac{\log n_{\min}}{n_{\min}} \approx \tfrac{\log \log n}{\log n} 
\]
which establishes the conditions of Theorem \ref{thm:convex_recovery2}. 
On the other hand, the last condition of Theorem \ref{thm:convex_recovery} is not satisfied as one cannot find a constant value $\alpha>0$ for which 
\[
\sum_{k=1}^r n_k^\alpha = \left( \tfrac{n}{\log n}- m \sqrt{\tfrac{n}{\log n}}  \right) \log^{-\alpha}n
+ m (n\log n)^{-\alpha/2}
\]
is $o(1)$ while $n$ grows. 
%---------------------------------------------------------------------------------------------------------
\paragraph{Example \ref{ex:hard}:}
For the first configuration, Theorem \ref{thm:convex_recovery} requires 
$
f^2(n) \gtrsim \max\{ \tfrac{\log n_1}{n_1} \,,\, \tfrac{\log n_{\min}}{n_{\min}}   \,,\, \tfrac{n}{n_1^2} \}
$ 
while Theorem \ref{thm:convex_recovery2} requires 
$
f^2(n) \gtrsim \max\{ \tfrac{\log n_1}{n_1} \,,\, \tfrac{\log n}{n_{\min}}   \,,\, \tfrac{n}{n_1^2} \}
$ 
and both require $n_{\min}\gtrsim \sqrt{n}\,$. Therefore, both set of requirements can be written as
\[
f^2(n) \gtrsim \max\{ \tfrac{\log n}{n_{\min}}   \,,\, \tfrac{n}{n_1^2} \}\quad,\quad n_{\min}\gtrsim \sqrt{n}\,.
\]
%---------------------------------------------------------------------------------------------------------

%\section{Additional impossibility conditions}
%Additional impossibility conditions: If any of the following conditions holds, then it is impossible to recover:
%
%$n\geq e\, n_{\max}\,$, $n_{\min}\geq 2$ and 
%\[
%2\sum_i n_i^2p_i\log \tfrac{e n}{n_i} +2\log \tfrac{1-  p_{\min}}{1-  p_{\max}}  
%\leq \tfrac{1}{2}\sum_i n_i\log \tfrac{n}{n_i}-r-2 
%\]
%\[
%\tfrac{1}{2}r + \log \tfrac{1-  p_{\min}}{1-  p_{\max}}  +1 + \sum_i n_i^2p_i
%\leq \sum_i \left(\tfrac{1}{4} -  n_i p_i \right) n_i \log \tfrac{n}{n_i} 
%\]
%\[
%\tfrac{1}{2}r + \log \tfrac{1-  p_{\min}}{1-  p_{\max}}  +1 + \sum_i n_i^2p_i
%\leq (\tfrac{1}{4} - \sum n_i^2p_i) \log n + \sum (n_ip_i-\tfrac{1}{4})n_i \log n_i
%\]
%or
%\[
%\tfrac{2}{\sqrt{e}} n \sum_i n_ip_i +2\log \tfrac{1-  p_{\min}}{1-  p_{\max}}  
%\leq \tfrac{1}{2}\sum_i n_i\log \tfrac{n}{n_i}-r-2 \,.
%\]
%Similarly, if $n_i^2p_i\leq n^2q$ and
%\[
%2\sum_i n_i^2(p_i-q)\log \tfrac{e n}{n_i} +2\log \tfrac{1-  p_{\min}}{1-  p_{\max}}  
%\leq \tfrac{1}{2}\sum_i n_i\log \tfrac{n}{n_i}-r-2 
%\]
%or 
%\[
%\tfrac{2}{\sqrt{e}} n \sum_i \rho_i +2\log \tfrac{1-  p_{\min}}{1-  p_{\max}}  
%\leq \tfrac{1}{2}\sum_i n_i\log \tfrac{n}{n_i}-r-2 \,.
%\]

\end{document}